\documentclass[12pt,draftclsnofoot,onecolumn]{IEEEtran}
\pdfoutput=1
\usepackage{amsmath}
\usepackage[backref=page, hidelinks]{hyperref}
\usepackage{graphicx}
\usepackage{color}

\usepackage{cite}
\usepackage{amsthm}
\usepackage{amssymb,amsfonts}
\usepackage{algorithmic}
\usepackage{textcomp}
\usepackage{float}
\usepackage[mathscr]{eucal}
\usepackage{ifpdf}
\usepackage{units}
\usepackage{setspace}
\usepackage{algorithm}
\usepackage{array}
\usepackage{xcolor}
\usepackage{multirow}
\usepackage{enumerate}
\usepackage{mathtools}
\usepackage{footnote}
\newtheorem{theorem}{Theorem}
\newtheorem{lemma}{Lemma}

\newtheorem{corollary}{Corollary}
\usepackage{graphics}
\usepackage{bbm}

\usepackage{tikz}
\usetikzlibrary{arrows,calc}

\usepackage{scalerel}
\usepackage{subcaption}
\usepackage{wrapfig,booktabs}
\allowdisplaybreaks
\begin{document}
\title{Non-stationary Delayed Combinatorial Semi-Bandit with Causally Related Rewards
}
%
%
\author{
\IEEEauthorblockN{Saeed Ghoorchian and Setareh Maghsudi}
\thanks{The authors are with the Faculty of Mathematics and Natural Sciences, T{\"u}bingen University, 72074 T{\"u}bingen, Germany. S. M. is also with the Fraunhofer Heinrich Herz Institute, Berlin, Germany. 
E-mail: saeed.ghoorchian@uni-tuebingen.de, setareh.maghsudi@uni-tuebingen.de
}
}
%
%
%
\maketitle 
%
\begin{abstract}
 Sequential decision-making under uncertainty is often associated with long feedback delays. Such delays degrade the performance of the learning agent in identifying a subset of arms with the optimal collective reward in the long run. This problem becomes significantly challenging in a non-stationary environment with structural dependencies amongst the reward distributions associated with the arms. Therefore, besides adapting to delays and environmental changes, learning the causal relations alleviates the adverse effects of feedback delay on the decision-making process. We formalize the described setting as a non-stationary and delayed combinatorial semi-bandit problem with causally related rewards. We model the causal relations by a directed graph in a stationary structural equation model. The agent maximizes the long-term average payoff, defined as a linear function of the base arms' rewards. We develop a policy that learns the structural dependencies from delayed feedback and utilizes that to optimize the decision-making while adapting to drifts. We prove a regret bound for the performance of the proposed algorithm. Besides, we evaluate our method via numerical analysis using synthetic and real-world datasets to detect the regions that contribute the most to the spread of Covid-19 in Italy.
\end{abstract}
\section{Introduction}
\label{sec:Intro}
Optimizing the long-run accumulated payoffs is the core challenge of online decision-making. In real-world scenarios, the learner often receives feedback with long delays and performs the learning task in a frequently-varying environment. For example, researchers have recently attempted to use the collected data to analyze the Covid-19 spread within a country \cite{Mastakouri20:CAC, Bastani21:EAT, Nourani22:LCS}. In this example, the testing results become available only after a while, thereby delaying the received information. Moreover, the average number of individuals infected within a region changes over time due to several factors, such as that region's geographical- and demographical characteristics. Such changes render the spread pattern of Covid-19 disease difficult to understand. This problem becomes aggravated when considering mobility amongst different regions. Such mobility results in causal relations amongst the total daily new cases of regions which in turn affects the trend of daily infected cases of each region.  

The challenges mentioned above call for a suitable framework to efficiently model and solve the problem. We take advantage of the Multi-Armed Bandit (MAB) problem \cite{robbins1952some}, where an agent sequentially chooses an arm and the environment reveals feedback drawn from some unknown distribution. The agent's goal is to maximize the cumulative reward over a finite time horizon. Alternatively, the objective is to minimize long-term regret, which is the difference between the accumulated reward of the optimal policy in hindsight and that of the agent's decision-making policy. In this scenario, the agent experiences the exploration-exploitation dilemma, where the decision has to be made between exploring options to acquire new knowledge and selecting an option by exploiting the existing knowledge \cite{Maghsudi16:IWC}. Our model is related to combinatorial semi-bandit \cite{chen2013combinatorial} where the agent is allowed to select a super arm, i.e., a subset of base arms, at each round of decision-making. In this setting, the agent observes a base arm's reward if it belongs to the selected super arm. Consequently, the agent accumulates the collective reward associated with the selected super arm.

We model the described problem using the combinatorial bandit setting and introduce the non-stationary delayed combinatorial semi-bandit problem with causally related rewards, which we refer to as NDC bandit for short. In this problem, we use Structural Equation Models (SEMs) \cite{kaplan2008structural} to model the causal relations. The underlying causal structure that affects the rewards is unknown to the agent. The nodal observation in the graph signal consists of the instantaneous reward of the corresponding base arm and an additional term resulting from the causal influences of other base arms' rewards. In our framework, the agent aims to maximize the long-term average payoff, defined as a linear function of the base arms' rewards and dependent on the network topology.

We propose and analyze an algorithm to solve the NDC bandit problem. Our proposed policy consists of two learning phases at each round of decision-making; first, the agent determines the causal relations by learning the network's topology while taking into account the delayed feedback. Second, the agent exploits the learned graph to improve the decision-making process while coping with abrupt changes in the environment. To this end, it utilizes a discount factor to reduce the influence of past observations with time. We prove a regret bound for the performance of our algorithm. The numerical results on synthetic data demonstrate our algorithm's superiority over several benchmarks. In addition to our experiments with synthetic data, we apply our method to analyze the development of Covid-19 in Italy. We employ our method to detect the regions that contribute the most to the spread of Covid-19 in the country while assuming that the testing results are delayed, and the environment is non-stationary.
\subsection{Related Works}
\label{subsec:Intro}
Most real-world problems are non-stationary in their nature. Bandit-based algorithms developed for non-stationary online learning problems, such as \cite{Garivier11:OUC, Xu20:CBB, Hariri15:ATU, Russac19:WLB, Zeng16:OCA, Ghoorchian21:MAB, Chen20:CSB}, inherently rely on the availability of recent feedback without delay. However, learners in many real-world problems are often limited in accessing such immediate feedback; such limitation arises due to a delay in receiving feedback, which badly affects the performance of the aforementioned methods. In addition to the delay, having causal dependencies in the system \cite{lattimore2016causal, Nourani22:LCS} makes it hard to adapt to environmental changes using the above-mentioned algorithms.

Online learning with delayed feedback has been investigated both in the full feedback setting \cite{Joulani13:OLU, Agarwal11:DDS} and partial feedback setting \cite{Travis15:TQM, Cesa-bianchi16:DAC}. The proposed algorithms only start learning after having received enough feedback from the environment. Consequently, such methods are effective in stationary environments. However, in a non-stationary environment where system parameters undergo abrupt changes, the aforementioned methods are not appropriate anymore. In the worst-case scenario, if the environment changes in the number of rounds less than or equal to the length of feedback delay, it is not possible to perform the learning task, as, by the time the learner receives the information, it loses its value. To address this problem, the authors in \cite{Vernade20:NSD} disentangle the effects of delays and non-stationarity by introducing intermediate signals that become available to the learner without delay. In the proposed method, the authors assume that, given the intermediate signals, the system's long-term behavior is stationary. However, the authors do not consider the possible causal dependencies amongst the arms' reward distributions.

The combinatorial bandit problem is well-investigated in the literature \cite{Chen13:CMA, Chen20:CSB, huyuk2019analysis, chen2016combinatorial, tang2017networked, yu2020graphical}. For example, \cite{huyuk2019analysis} considers a combinatorial semi-bandit problem with probabilistically triggered arms, where selected super arms can probabilistically trigger other base arms. The authors propose the combinatorial Thompson sampling algorithm to solve the problem. At each decision-making time, the algorithm uses the entire collected feedback up to the current time and an oracle to select the best combinatorial action. Similarly, \cite{chen2016combinatorial} studies the combinatorial semi-bandit problem with probabilistically triggered arms and propose an Upper Confidence Bound (UCB)-based algorithm. The proposed algorithm uses an oracle to select a super arm at each time by using the entire observed data up to the current time. In \cite{tang2017networked}, the authors consider a combinatorial setting where at each round of play, the agent receives the reward of the selected super arm and some side rewards from the selected base arms' neighbors. The proposed method exploits the prior knowledge of statistical structures to learn the best combinatorial strategy. Compared to the aforementioned works, our proposed algorithm can work with delayed feedback and adapts to changes in the environment. Moreover, it learns the underlying causal structure over time and exploits it to improve the decision-making process. Hence, in our proposed framework, we do not require prior knowledge of the structural dependencies.

The remaining literature that studies the underlying structure of the problem is not suitable to deal with delayed feedback in changing environments. For example, the authors in \cite{toni2018spectral} attempt to learn the structure of a combinatorial bandit problem with i.i.d. rewards. In the considered setting, there is neither a delay in receiving feedback nor causal relations between rewards. Moreover, \cite{sen2017identifying} employs the MAB framework to identify the best soft intervention on a causal system, while it is assumed that the causal graph is only partially unknown. The authors assume a stationary environment and do not consider possible delays in receiving feedback. Our work is most closely related to \cite{Nourani22:LCS}, where the authors model the causal relations by a directed graph in a stationary SEM. However, the proposed framework ignores the changes in the environment and is not able to work with delayed feedback.

%
The rest of the paper is organized as follows. We formulate the NDC bandit problem in Section \ref{sec:ProFor}. In Section \ref{sec:Strategy}, we propose our algorithm, namely NDC-SEM, and theoretically analyze its regret performance in Section \ref{sec:RegAnalysis}. In Section \ref{sec:NumAnalysis}, we present the results of numerical analysis. Section \ref{sec:Conclusion} concludes the paper.
\section{Problem Formulation}
\label{sec:ProFor}
We consider a causally structured combinatorial semi-bandit problem with $N$ \textit{base arms} gathered in the set $[N] = \{1, 2, \dots, N\}$. Let $\mathbf{b}_{t} = [\mathbf{b}_{t}[1], \mathbf{b}_{t}[2], \dots, \mathbf{b}_{t}[N]] \in [0,1]^{N}$ represent the vector of \textit{instantaneous rewards} of the base arms at time $t$. Moreover, by $\boldsymbol{\beta}_{t} = [\boldsymbol{\beta}_{t}[1], \boldsymbol{\beta}_{t}[2], \dots, \boldsymbol{\beta}_{t}[N]]$, we denote the expected instantaneous reward vector of the base arms at time $t$. For each base arm $i \in [N]$, the instantaneous rewards $\mathbf{b}_{t}[i]$ over time are independent random variables, drawn from an unknown probability distribution with mean $\boldsymbol{\beta}_{t}[i]$.  

We model the causal relationships in the system by using an unknown stationary sparse Directed Acyclic Graph (DAG) $\mathcal{G} = (\mathcal{V}, \mathcal{E}, \mathbf{A})$. $\mathcal{V}$ denotes the set of $N$ vertices, i.e., $\left | \mathcal{V} \right | = N$, $\mathcal{E}$ represents the edge set, and $\mathbf{A}$ is the weighted adjacency matrix. Moreover, we use $p \leq N - 1$ to denote the length of the longest path in the graph $\mathcal{G}$. The reward generating processes in the bandit setting follow an error-free Structural Equation Model (SEM) \cite{giannakis2018topology, Bazerque13:IOS}. At each time $t$, we use $\mathbf{z}_{t} = [\mathbf{z}_{t}[1], \mathbf{z}_{t}[2], \dots, \mathbf{z}_{t}[N]]$ and $\mathbf{y}_{t} = [\mathbf{y}_{t}[1], \mathbf{y}_{t}[2], \dots, \mathbf{y}_{t}[N]]$ to denote the exogenous input vector and the endogenous output vector of the SEM, respectively. We refer to $\mathbf{z}_{t}$ and $\mathbf{y}_{t}$ as the \textit{feedback} from the environment at time $t$.

\textbf{Game Protocol:}
At each time $t$, the sequence of the events in the NDC bandit problem is as follows:
%
    (i) The agent determines a \textit{super arm}, i.e., a subset of base arms, by choosing a \textit{decision vector} $\mathbf{x}_{t} = [\mathbf{x}_{t}[1], \mathbf{x}_{t}[2], \dots, \mathbf{x}_{t}[N]] \in \left \{ 0,1 \right \}^{N}$, where $\mathbf{x}_{t}[i] = 1$ if the agent selects the base arm $i$ and $\mathbf{x}_{t}[i] = 0$ otherwise. At each time of play, the agent selects at most $s$ base arms, where the sparsity parameter $s$ is pre-determined and known.
    %
    (ii) After a possibly random delay $D_{t}$, the environment reveals the feedback $\mathbf{z}_{t}$ and $\mathbf{y}_{t}$ to the agent. For simplicity, we assume throughout the paper that the delays are constant, i.e., $\forall t$, $D_{t} = D$; our results can be extended to random delays.
%
The environment presumably changes over time. To model the non-stationarity in the environment, we assume that there exist $\Upsilon_{T}$ time instants before a time horizon $T$ where at least one of the expected rewards $\boldsymbol{\beta}_{t}[i]$, for any $i \in [N]$, changes abruptly.

In \textbf{Fig. \ref{fig:ExmpModel}}, we depict an exemplary graph with four nodes and the underlying causal relations. Note that there does not exist necessarily a causal relation between every pair of nodes. Based on our proposed model, at each time $t$, the agent observes both the exogenous input vector $\mathbf{z}_{t-D}$ and the endogenous output vector $\mathbf{y}_{t-D}$ for the time $t-D$.

\textbf{Expected Regret:}
We define the exogenous input $\mathbf{z}_{t}$ at time $t$ as
\begin{align}
\label{eq:ExoInput}
    \mathbf{z}_{t} = \operatorname{diag}(\mathbf{b}_{t}) \mathbf{x}_{t},
\end{align}
where $\operatorname{diag}(\cdot)$ represents the operator that diagonalizes its given input vector. The exogenous input $\mathbf{z}_{t}$ represents the semi-bandit feedback at time $t$ of the decision-making problem. Accordingly, for each $i \in [N]$, we define the endogenous output $\mathbf{y}_{t}[i]$ as
\begin{equation}
\label{eq:OverallRew}
\mathbf{y}_{t}[i] = \sum_{i \neq j} \mathbf{A}[i,j] \mathbf{y}_{t}[j] + \mathbf{F}[i,i] \mathbf{z}_{t}[i],~~~~\forall i \in [N],
\end{equation}
where $\mathbf{F}$ is a diagonal matrix that captures the effects of the exogenous input vector $\mathbf{z}_{t}$. The SEM in (\ref{eq:OverallRew}) implies that $\mathbf{y}_{t}[i]$ depends on the exogenous input signal $\mathbf{z}_{t}[i]$ as well as the endogenous outputs of single-hop neighbors. The endogenous output $\mathbf{y}_{t}[i]$ represents the \textit{overall reward} of the corresponding base arm $i \in [N]$ at time $t$. Hence, at each time $t$, the overall reward of each base arm consists of (i) a part that directly results from its instantaneous reward and (ii) another part that reflects the effect of causal influences of other base arms' overall rewards.

Based on (\ref{eq:OverallRew}), the base arms' overall rewards are causally related. The adjacency matrix $\mathbf{A}$ represents the causal relationships between the overall rewards; the element $\mathbf{A}[i,j]$ of the adjacency matrix denotes the causal impact of the overall reward of base arm $j$ on the overall reward of base arm $i$, and we have $\mathbf{A}[i,i] = 0$, $\forall i = 1, 2, \dots, N$. In our problem, the adjacency matrix $\mathbf{A}$ is unknown a priori, which means that the agent does not know the causal relationships between the base arms' overall rewards. The matrix form of (\ref{eq:OverallRew}) is defined as 
\begin{equation}
\label{eq:OverallRew-MatrixForm}
    \mathbf{y}_{t} = \mathbf{A} \mathbf{y}_{t} + \mathbf{F}\mathbf{z}_{t}.
\end{equation}
%

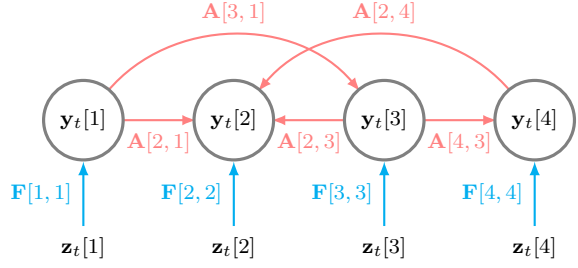
\begin{figure}[t]
\centering
\begin{tikzpicture}[-latex,scale=1, every node/.style={scale=1}, auto, node distance=2cm and 5cm, semithick,
state/.style = {circle, very thick, color=gray!100, fill=white, draw, text=black, minimum width =0.1 cm}] 

\node[state](q1)  {\tiny $\mathbf{y}_{t}[1]$};
\node[state, right of=q1] (q2) {\tiny $\mathbf{y}_{t}[2]$};
\node[state, right of=q2] (q3) {\tiny $\mathbf{y}_{t}[3]$};
\node[state, right of=q3] (q4) {\tiny $\mathbf{y}_{t}[4]$};

\filldraw[red!50, thick] (q1) edge[bend left = 0] node[below] {\scriptsize \textit{$\mathbf{A}{[2,1]}$}} (q2);
\filldraw[red!50, thick] (q3) edge[bend left= 0] node {\scriptsize \textit{$\mathbf{A}{[2,3]}$}} (q2);
\filldraw[red!50, thick] (q1) edge[bend left= +50] node {\scriptsize \textit{$\mathbf{A}{[3,1]}$}} (q3);
\filldraw[red!50, thick] (q4) edge[bend left = -50] node[above] {\scriptsize \textit{$\mathbf{A}{[2,4]}$}} (q2);
\filldraw[red!50, thick] (q3) edge[bend left = 0] node[below] {\scriptsize \textit{$\mathbf{A}{[4,3]}$}} (q4);




\path (0,-1.7) node (x) {\scriptsize $\mathbf{z}_{t}[1]$}
(0,0) node[state](q1) {\scriptsize $\mathbf{y}_{t}[1]$};
\filldraw[cyan!100, thick] (x) -- node[midway] {\scriptsize $\mathbf{F}[1,1]$} (q1);

\path (2,-1.7) node (x) {\scriptsize $\mathbf{z}_{t}[2]$}
(2,0) node[state](q2) {\scriptsize $\mathbf{y}_{t}[2]$};
\filldraw[cyan!100, thick] (x) -- node[midway] {\scriptsize $\mathbf{F}[2,2]$} (q2);

\path (4,-1.7) node (x) {\scriptsize $\mathbf{z}_{t}[3]$}
(4,0) node[state](q3) {\scriptsize $\mathbf{y}_{t}[3]$};
\filldraw[cyan!100, thick] (x) -- node[midway] {\scriptsize $\mathbf{F}[3,3]$} (q3); 

\path (6,-1.7) node (x) {\scriptsize $\mathbf{z}_{t}[4]$}
(6,0) node[state](q4) {\scriptsize $\mathbf{y}_{t}[4]$};
\filldraw[cyan!100, thick] (x) -- node[midway] {\scriptsize $\mathbf{F}[4,4]$} (q4);
\end{tikzpicture}
\caption{An exemplary illustration of a graph with $4$ nodes and the corresponding causal relations. The red directed edges represent the causal relationships within the network.}
\label{fig:ExmpModel}
\end{figure}

By solving (\ref{eq:OverallRew-MatrixForm}) for variable $\mathbf{y}_{t}$ and using (\ref{eq:ExoInput}) in place of $\mathbf{z}_{t}$, we achieve
\begin{equation}
\label{eq:OverallRew-Solved}
    \mathbf{y}_{t}=(\mathbf{I-A})^{-1}\mathbf{F}\operatorname{diag}(\mathbf{b}_{t}) \mathbf{x}_{t}.
\end{equation}
Therefore, we define the \textit{payoff} at time $t$, upon choosing the decision vector $\mathbf{x}_{t}$ by the agent, as
\begin{equation}
\label{eq:Payoff}
    r_{t}(\mathbf{x}_{t}) = {\bf 1}^{\top}\mathbf{y}_{t} = {\bf 1}^{\top} (\mathbf{I-A})^{-1} \mathbf{F} \operatorname{diag}(\mathbf{b}_{t}) \mathbf{x}_{t}, 
\end{equation}
where ${\bf 1}$ is the $N$-dimensional vector of ones. Note that the matrix $(\mathbf{I-A})$ is invertible due to the fact that the graph $\mathcal{G}$ is a DAG, which implies that with a proper indexing of the vertices, the adjacency matrix $\mathbf{A}$ is a strictly upper triangular matrix. In our problem, since the agent directly observes the exogenous input, we assume that the effects of $\mathbf{F}$ on the exogenous input is already integrated in the instantaneous rewards. Hence, to simplify the notation and without loss of generality, we assume that $\mathbf{F} = \mathbf{I}$ in the following.

Finally, at time $t$, when the decision vector $\mathbf{x}_{t}$ is chosen by the agent, the expected payoff can be calculated as
\begin{equation}
\label{eq:ExpPayoff}
    \mu_{t}(\mathbf{x}_{t}) = \mathbb{E}\left [  r_{t}(\mathbf{X}) | \mathbf{X} = \mathbf{x}_{t} \right ],
\end{equation}
where the expectation is taken with respect to the randomness in the reward generating processes. 

The expected payoff defined in (\ref{eq:ExpPayoff}) shows that we are dealing with a linear combinatorial semi-bandit problem with causally related rewards in a non-stationary environment. Note that, for a fixed decision vector $\mathbf{x}$, the expected payoff may change over time due to the possible changes in the expected value of base arms' instantaneous rewards. In addition, due to the randomness in selection of the decision vector $\mathbf{x}_{t}$, the consecutive overall reward vectors $\mathbf{y}_{t}$ become non-identically distributed. 

Let $\mathcal{X} = \left \{ \mathbf{x} \mid \mathbf{x} \in \{0,1\}^{N} \wedge \left \| \mathbf{x} \right \|_{0}  \leq s \right\}$ denote the set of feasible decision vectors, %
%
where $\left\| \cdot \right \|_{0}$ determines the number of non-zero elements in a given vector. Ideally, the agent maximizes the expected accumulated payoff over the time horizon $T$. Alternatively, the agent minimizes the expected regret, i.e., the difference between the expected accumulated payoff of an oracle that follows the optimal policy and that of the agent that follows the applied policy. We define the expected regret as
\begin{equation}
\label{eq:ExpRegret}
    \mathcal{R}_{T}(\mathcal{X}) =  \sum_{t = 1}^{T} [\mu_{t}(\mathbf{x}_{t}^{\ast}) - \mu_{t}(\mathbf{x}_{t})],
\end{equation}
where $\mathbf{x}_{t}^{*} = \text{argmax}_{\mathbf{x} \in \mathcal{X}} ~\mu_{t}(\mathbf{x})$ and $\mathbf{x}_{t}$ denote the optimal decision vector and the selected decision vector under the applied policy at time $t$, respectively.

\section{Decision-Making Strategy}
\label{sec:Strategy}
This section presents our decision-making strategy to minimize the expected regret defined in (\ref{eq:ExpRegret}). Note that the expected payoff defined in (\ref{eq:ExpPayoff}) implies that the knowledge of $\mathbf{A}$ and $\boldsymbol{\beta}_{t}$ are essential to select the best decision vectors that maximize the accumulated payoffs. Hence, our proposed algorithm estimates them before making decisions. More precisely, our proposed policy consists of two learning components: (i) an online graph learning using delayed feedback and (ii) an adaptive Upper Confidence Bound (UCB)-based reward learning. In the following, we describe each component separately and propose our algorithm, namely NDC-SEM.
\subsection{Online Graph Learning under Delayed Feedback}
\label{subsec:GraphLearning}
In our proposed policy, the agent attempts to learn the causal relations; nonetheless, not the entire feedback becomes immediately available. In the following, we develop an online graph learning framework that uses the delayed feedback, i.e., the delayed exogenous input and endogenous output vectors, to estimate the adjacency matrix $\mathbf{A}$.

At each time $t$, due to the existing delay $D$, the agent only observes the feedback up to the time $t - D$. Therefore, at time $t$, we collect the received feedback in $\mathbf{Z}_{t}^{D} = [\mathbf{z}_{1} \hdots \mathbf{z}_{t-D}]$ and $\mathbf{Y}_{t}^{D} = [\mathbf{y}_{1} \hdots \mathbf{y}_{t-D}]$. Then,
\begin{equation}
\label{eq:OveralRew-2xMatrixForm}
    \mathbf{Y}_{t}^{D} = \mathbf{A} \mathbf{Y}_{t}^{D} + \mathbf{Z}_{t}^{D}.
\end{equation}
We assume that the right indexing of the vertices is known prior to estimating the ground truth adjacency matrix. At each time $t$, we exploit the received feedback $\mathbf{Y}_{t}^{D}$ and $\mathbf{Z}_{t}^{D}$ as the input to a parametric graph learning algorithm \cite{giannakis2018topology, dong2019learning}. Formally, at time $t$, we use the following optimization problem to estimate the adjacency matrix.
\begin{equation}
\label{eq:Optimization-A}
\begin{aligned}
\hat{\mathbf{A}}_{t} = \underset{\mathbf{A}}{\text{argmin}}~ &\left \| \mathbf{Y}_{t}^{D} - \mathbf{A} \mathbf{Y}_{t}^{D} - \mathbf{Z}_{t}^{D} \right \|_{2}^{2} + \lambda \left \| \mathbf{A} \right \|_{1} \\
\textrm{s.t.} ~~~ &\mathbf{A}[i,j] \geq 0, ~ \forall i, j, \\ 
&\mathbf{A}[i,j]=0, ~ \forall i \geq j,
\end{aligned}
\end{equation}
where $\left \| \cdot \right \|_{2}$ and $\left \| \cdot \right \|_{1}$ represent the $L^{2}\text{-norm}$ and $L^{1}\text{-norm}$ of matrices, respectively. Moreover, $\lambda$ is the regularization parameter. The regularization term in (\ref{eq:Optimization-A}) imposes the sparsity property on the estimated matrix $\hat{\mathbf{A}}_{t}$. In addition, it guarantees that the optimization problem (\ref{eq:Optimization-A}) is convex.
\subsection{Adaptive Decision Vector Selection}
\label{subsec:Alg}
Our proposed decision-making policy is presented in \textbf{Algorithm \ref{Alg:NDC-SEM}}. Our decision-making strategy relies on confidence regions for rewards. Moreover, it adapts to changes in the environment by using a discount factor $\gamma \in (0, 1)$ when estimating the expected value of base arms' instantaneous rewards. The discount factor $\gamma$, given as input to the algorithm, helps to reduce the influence of observations with time; by using the discount factor, the agent gives more importance to recent observations relative to those in the distant past. Formally, for each base arm $i \in [N]$ at time $t$, we define
%
\begin{equation}
    \label{eq:AvgRew}
    \hat{\boldsymbol{\beta}}_{t}[i] = \frac{\sum_{\tau = 1}^{t-D} \gamma^{t - \tau} \mathbf{b}_{\tau}[i] \mathbbm{1}\left\{ \mathbf{x}_{\tau}[i] = 1\right\} }{\mathbf{M}_{t}^{\gamma,D}[i]},
\end{equation}
where $\mathbf{M}_{t}^{\gamma,D}[i] = \sum_{\tau = 1}^{t-D} \gamma^{t - \tau} \mathbbm{1}\left\{ \mathbf{x}_{\tau}[i] = 1\right\}$.
%
%

\begin{algorithm}[t]
\caption{NDC-SEM for NDC bandits with Structural Equation Models.}
\label{Alg:NDC-SEM}
\textbf{Input}: Sparsity parameter $s$, discount factor $\gamma$, initialization matrix $\mathbf{H}$. \\
\vspace{-4mm}
\begin{algorithmic}[1] 
\FOR{$t = 1, \dots, N$}
    \STATE Select column $t$ of the initialization matrix $\mathbf{H}$ as the decision vector $\mathbf{x}_{t}$.
    \STATE Receive feedback $\mathbf{z}_{t-D}$ and $\mathbf{y}_{t-D}$ for $t > D$.
\ENDFOR
\FOR{$t = N+1, \dots, T$}
    \STATE Obtain $\hat{\mathbf{A}}_{t-1}$ by solving (\ref{eq:Optimization-A}).
    %
    \STATE Calculate $\mathbf{E}_{t-1}[i]$ using (\ref{eq:UCB}), $\forall i \in [N]$.
    \STATE Select decision vector $\mathbf{x}_{t}$ that solves (\ref{eq:Optimization-x}).
    \STATE Receive feedback $\mathbf{z}_{t-D}$ and $\mathbf{y}_{t-D}$ for $t > D$.
    %
\ENDFOR
\end{algorithmic}
\end{algorithm}

In the initialization phase, NDC-SEM algorithm uses an upper-triangular \textbf{initialization matrix} $\mathbf{H} \in \{0, 1\}^{N \times N}$. At each time $t$ during the first $N$ times of play, NDC-SEM selects the column $t$ of $\mathbf{H}$ as the corresponding decision vector. We create the matrix $\mathbf{H}$ as follows. All diagonal elements of $\mathbf{H}$ are equal to $1$. As for the column $i$, if $i \leq s$, we set all elements above diagonal to $1$. If $s + 1 \leq i \leq N$, we select $s-1$ elements above diagonal uniformly at random and set them to $1$. The remaining elements are set to $0$. Such a specific strategy in the initialization phase creates rich data that helps to learn the ground truth adjacency matrix. In addition, it guarantees that all the base arms are pulled at least once, and the matrix $\mathbf{H}$ is full rank. Consequently, the adjacency matrix $\mathbf{A}$ is uniquely identifiable from the collected feedback \cite{Bazerque13:IOS}.

In the next phase, the NDC-SEM algorithm takes two consecutive steps at each time $t$ to learn the causal relationships and the expected instantaneous rewards of the base arms. In the first step, it uses the collected delayed feedback $\mathbf{Y}_{t}^{D}$ and $\mathbf{Z}_{t}^{D}$ to estimate the adjacency matrix by solving the optimization problem (\ref{eq:Optimization-A}). In the second step, it uses the reward observations to calculate the UCB index $\mathbf{E}_{t}[i]$ for each base arm $i$, defined as
\begin{equation}
    \label{eq:UCB}
    \mathbf{E}_{t}[i] = \hat{\boldsymbol{\beta}}_{t}[i] + 2 \sqrt{\frac{\xi (s+1) \log{m_{t}^{\gamma}}}{\mathbf{M}_{t}^{\gamma,D}[j]}},
\end{equation}
where $\xi$ is a tunable parameter that controls the exploration power of the algorithm and $m_{t}^{\gamma} = \sum_{\tau = 1}^{t} \gamma^{t-\tau}$.
Afterward, the NDC-SEM algorithm selects a decision vector $\mathbf{x}_{t}$ using the current estimate of the adjacency matrix and the developed UCB indices of the base arms. Let $\mathbf{E}_{t} = [\mathbf{E}_{t}[1], \mathbf{E}_{t}[2], \dots, \mathbf{E}_{t}[N]]$. At time $t$, it selects $\mathbf{x}_{t}$ as
\begin{equation}
\label{eq:Optimization-x}
\mathbf{x}_{t} = \underset{\mathbf{x} \in \mathcal{X}}{\text{argmax}} \quad \mathbf{1}^{\top} (\mathbf{I} - \hat{\mathbf{A}}_{t-1})^{-1} \operatorname{diag}(\mathbf{E}_{t-1}) \mathbf{x}
\quad \textrm{s.t.} \quad \left \| \mathbf{x} \right \|_{0} \leq s.
\end{equation}

The fundamental aspect of our algorithm is that it works with delayed observations for each base arm rather than the delayed payoff observations for each super arm. As the same base arm can be included in different selected super arms, we can use the information obtained from selecting a super arm to improve our payoff estimation of other relevant super arms. This, combined with the fact that our algorithm adapts to non-stationary rewards and simultaneously learns the adjacency matrix, significantly speeds up the learning process, resulting in high performance for our proposed algorithm.

%
\section{Theoretical Analysis}
\label{sec:RegAnalysis}
In this section, we prove an upper bound on the expected regret of NDC-SEM algorithm. We use the following definitions in our regret analysis. Let $[T] = \{1, 2, \dots, T\}$. For any decision vector $\mathbf{x} \in \mathcal{X}$, let $\Delta_{t}(\mathbf{x}) = \mu_{t}(\mathbf{x}_{t}^{\ast}) - \mu_{t}(\mathbf{x})$. We define $\Delta_{\max} = \underset{t \in [T]}{\max}~ \underset{\mathbf{x}: \mu_{t}(\mathbf{x}) < \mu_{t}(\mathbf{x}_{t}^{\ast})}{\max}~ \Delta_{t}(\mathbf{x})$ and $\Delta_{\min} = \underset{t \in [T]}{\min}~ \underset{\mathbf{x}: \mu_{t}(\mathbf{x}) < \mu_{t}(\mathbf{x}_{t}^{\ast})}{\min}~ \Delta_{t}(\mathbf{x})$. Moreover, let $\mathbf{w}_{t}^{\top} = {\bf 1}^{\top}(\mathbf{I} - \hat{\mathbf{A}}_{t-1})^{-1}\text{diag}(\mathbf{x}_{t})$. We define $w_{\max} = \underset{t}{\max} ~\underset{i}{\max} ~\mathbf{w}_{t}[i]$.

\begin{theorem}
\label{thm:1}
Let $\xi > \frac{1}{2(s+1)}$. The expected regret of NDC-SEM algorithm is upper bounded as
\begin{align} \nonumber
   \mathcal{R}_{T}(\mathcal{X}) \hspace{-0.5mm}
   &\leq \hspace{-0.5mm} {\Bigg[} 1 + J(\gamma) \Upsilon_{T} +  \lceil{T(1-\gamma)\rceil}
   \hspace{-0.5mm} \left( \left \lceil {\frac{16 \xi s^{2} w_{\max}^{2} (s+1) \log{m_{T}^{\gamma}}}{ \Delta_{\min}^{2} }} \right \rceil \hspace{-0.5mm} \gamma^{-\frac{1}{1-\gamma}} + D \hspace{-0.5mm} \right) 
   \\
   &\hspace{20mm}+ 2 s^{p} \left\lceil\frac{1}{1-\gamma}\right\rceil^{2s} 
   \left( \frac{1}{1-\gamma}
   + \left\lceil \frac{\log \frac{1}{1-\gamma}}{\log{(1+\eta)}} \right\rceil^{p} \frac{T (1-\gamma)^{p}}{(1-\gamma^{\frac{1}{1-\gamma}})^{p}} \right) {\Bigg]} N \Delta_{\max}.
\end{align}
\end{theorem}
\begin{proof}
See Appendix \ref{subsec:ProofThm1}.
\end{proof}
%
It is possible to extend our theoretical analysis in Theorem \ref{thm:1} for random delays. In this case, the only affected part of the proof is the bound (\ref{eq:lemma1hyp}) derived in Lemma \ref{lem:CounterBound} (See the proof in Appendix \ref{app:mainresults} for details). To bound this event, we assume the worst-case scenario and use the maximum delay over the entire time horizon $T$, i.e., $D_{\max} = \max_{t \in [T]} D_{t}$. Then, using Corollary \ref{cor:RandomDelay}, the bound (\ref{eq:lemma1hyp}) can be replaced by (\ref{eq:cor1hyp}) in our proof. Hence, the expected regret will be of order $O(D_{\max})$ with respect to the delay variable.
%
\section{Numerical Analysis}
\label{sec:NumAnalysis}
In this section, we present the results of numerical experiments to provide more insight into the impact of delay, non-stationarity, and structural dependencies on the performance of learning algorithms. We show that our proposed algorithm can mitigate these impacts by learning the causal relations from delayed feedback to improve the decision-making process while adapting to changes in the environment in an efficient way. We test our algorithm in different scenarios using synthetic and real-world datasets and compare it with state-of-the-art benchmark algorithms.

\textbf{Benchmark Policies:}
We compare NDC-SEM with two categories of combinatorial semi-bandit algorithms; those that are agnostic towards learning the causal relations and the one benchmark that learns the causal structure of the problem. The former category in our experiment includes \textbf{CUCB} \cite{chen2016combinatorial}, \textbf{CTS} \cite{huyuk2019analysis}, and \textbf{FTRL} \cite{zimmert2019beating}. At each time, the CUCB policy uses an approximation oracle that takes as input the calculated UCB index for base arms and outputs a super arm. The CTS policy utilizes the Thompson sampling and an oracle to select a super arm at each time of play. The CUCB and CTS algorithms are designed to work with i.i.d. random variables. Moreover, they are delay-agnostic. The FTRL policy relies on the method of Follow-the-Regularized-Leader to select a super arm at each time. In addition, it does not take the possible delays in observations into account. The latter category includes only \textbf{SEM-UCB} \cite{Nourani22:LCS} that learns the structural dependencies and exploits this knowledge to select a super arm at each time. It is a UCB-based algorithm and works based on the individual observations of base arms rather than the payoff observations of super arms as a whole. The SEM-UCB algorithm is specially designed for stationary environments. In addition, it is delay-agnostic. Finally, we also consider a \textbf{random} policy that selects a super arm uniformly at random at each time.

\subsection{Synthetic Dataset}
\label{subsec:SynData}
We start our experiments by assessing the performance of our algorithm on a synthetic dataset. This way, we have access to the oracle, and therefore, we can perform various analyses on our proposed method. More specifically, we can compare the selected decision vectors by NDC-SEM with the decisions made by the oracle to provide more insight into the effectiveness of our proposed method. The setting of our simulation is as follows.

\textbf{Experimental Setup:}
We create a weighted directed acyclic graph consisting of $N = 10$ nodes. The edge density of the ground truth graph is $0.09$. The non-zero elements of the adjacency matrix $\mathbf{A}$ are drawn from a continuous uniform distribution over $[0.4,0.7]$.  The instantaneous rewards $\mathbf{b}_{t}[i]$ for each base arm $i$ are drawn from a Bernoulli distribution with piece-wise constant mean $\boldsymbol{\beta}_{t}[i]$. We consider $\Upsilon_{T} = 3$ change points in the expected instantaneous rewards at times $\{1000, 2500, 4000\}$. In Appendix \ref{App:ExpInstRew}, 
we elaborate more on the settings of expected instantaneous rewards. As demonstrated in Section \ref{sec:ProFor}, we generate the vector of overall rewards according to the SEM in (\ref{eq:OverallRew}). The regularization parameter $\lambda$ is tuned by grid search over $[10^{-5},10^{6}]$. We evaluate the estimated adjacency matrix at each time $t$ by using the mean squared error defined as $\text{MSE} = \frac{1}{N^{2}} \left\| \mathbf{A} - \hat{\mathbf{A}}_{t} \right\|_{\text{F}}^{2}$, where $\left \| \cdot \right \|_{\text{F}}$ denotes the Frobenius norm.

For the results to be comparable, we apply all the benchmarks to the vector of overall reward $\mathbf{y}_{t}$ at each time $t$. If a benchmark requires $\mathbf{y}_{t}$ to be in $[0,1]$, we feed the normalized version of $\mathbf{y}_{t}$ to the corresponding algorithm. Finally, in our experiments, we choose the sparsity parameter $s= 4$, meaning that the algorithms can choose $4$ base arms at each time of play. We run the experiment for $T = 5000$ time steps and repeat the experiment by considering three different values for delay $D \in \{50, 200, 400\}$. We tune the discount factor for NDC-SEM and set it to $\gamma = 0.985$. \textbf{Table \ref{table:PolicyParams}} in Appendix \ref{app:AddInfo-Synthetic} lists all the tuned parameters used in our experiments.

\textbf{Regret Comparison:}
We run the algorithms using the aforementioned setup. In \textbf{Fig. \ref{fig:Reg-Synthetic}}, we depict the trend of cumulative expected regret over time for each policy for different choices of delay $D$. Here, the oracle receives the feedback without delay. As we see, NDC-SEM outperforms all the other policies and can comply faster with abrupt environmental changes. This is because NDC-SEM estimates the graph structure using the delayed feedback; hence, it has a better knowledge of the causal relationships in the network. Moreover, NDC-SEM uses a discount factor $\gamma$ to weight the observations when estimating the expected instantaneous rewards. Therefore, it has a smoother curve around change points, unlike other policies that jump suddenly. We emphasize that our algorithm can deal with delayed, causally related, and non-i.i.d. variables. This is a significant improvement over the previous methods that either do not consider delayed and non-i.i.d feedback or do not learn the causal relations.

\begin{figure*}[t!]
    \centering
    \hspace{-16mm}
     \begin{subfigure}[t]{0.3\textwidth}
         \centering
         \includegraphics[width=1.04\textwidth]{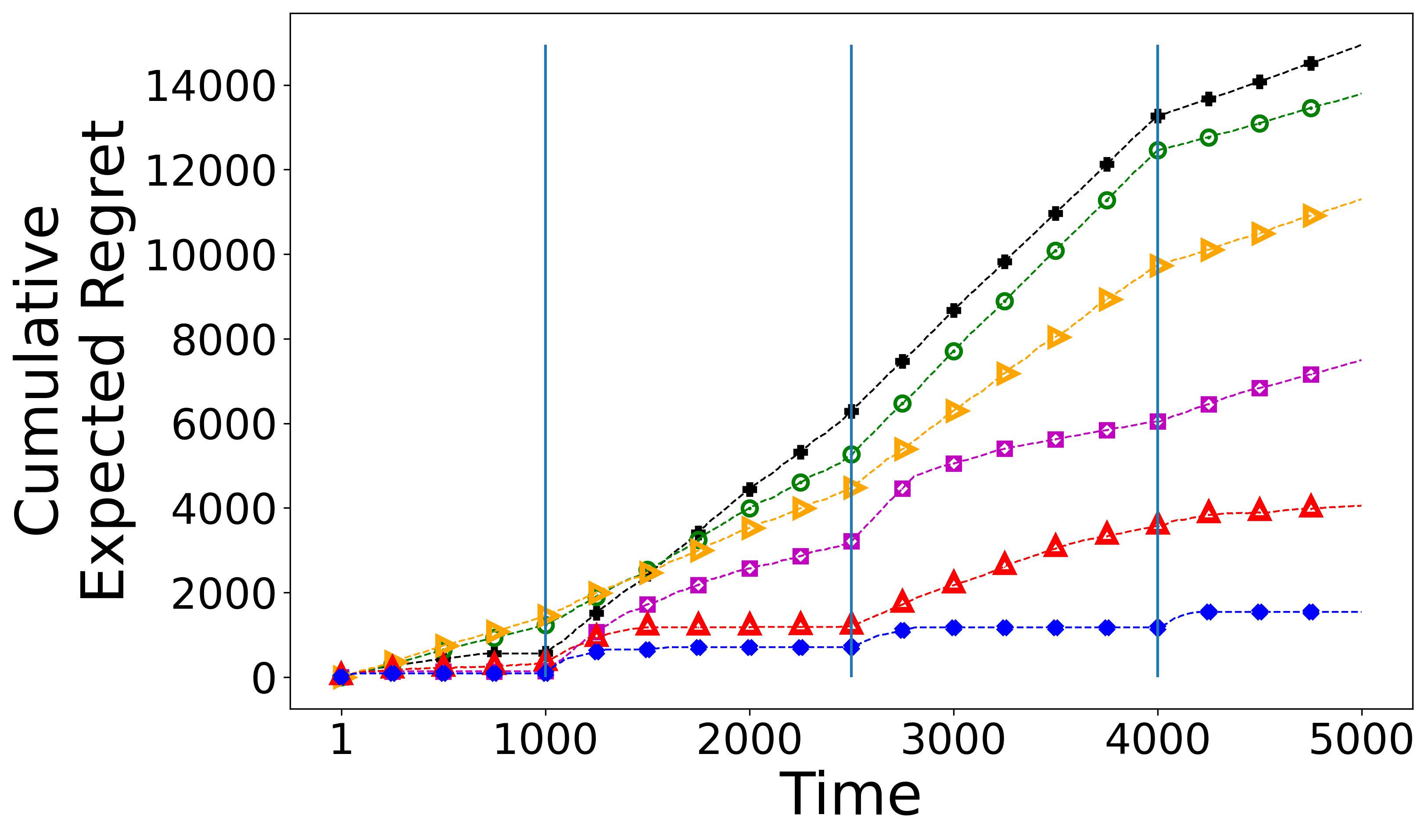}
         \label{subfig:delay1}
     \end{subfigure}
    \hspace{0mm}
     \begin{subfigure}[t]{0.3\textwidth}
         \centering
         \includegraphics[width=0.95\textwidth]{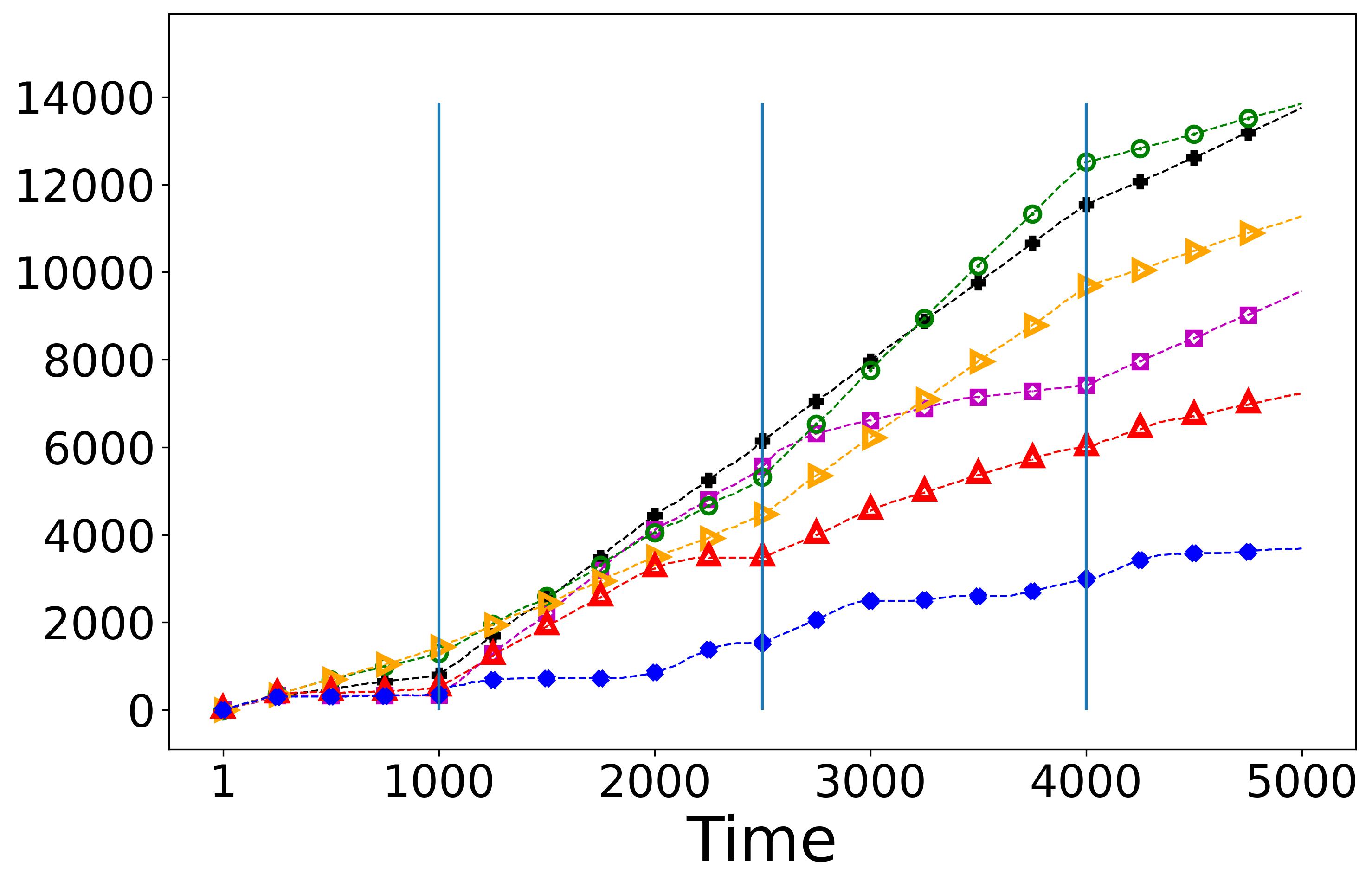}
         \label{subfig:delay2}
     \end{subfigure}
    \hspace{-3.5mm}
     \begin{subfigure}[t]{0.3\textwidth}
         \centering
         \includegraphics[width=1.29\textwidth]{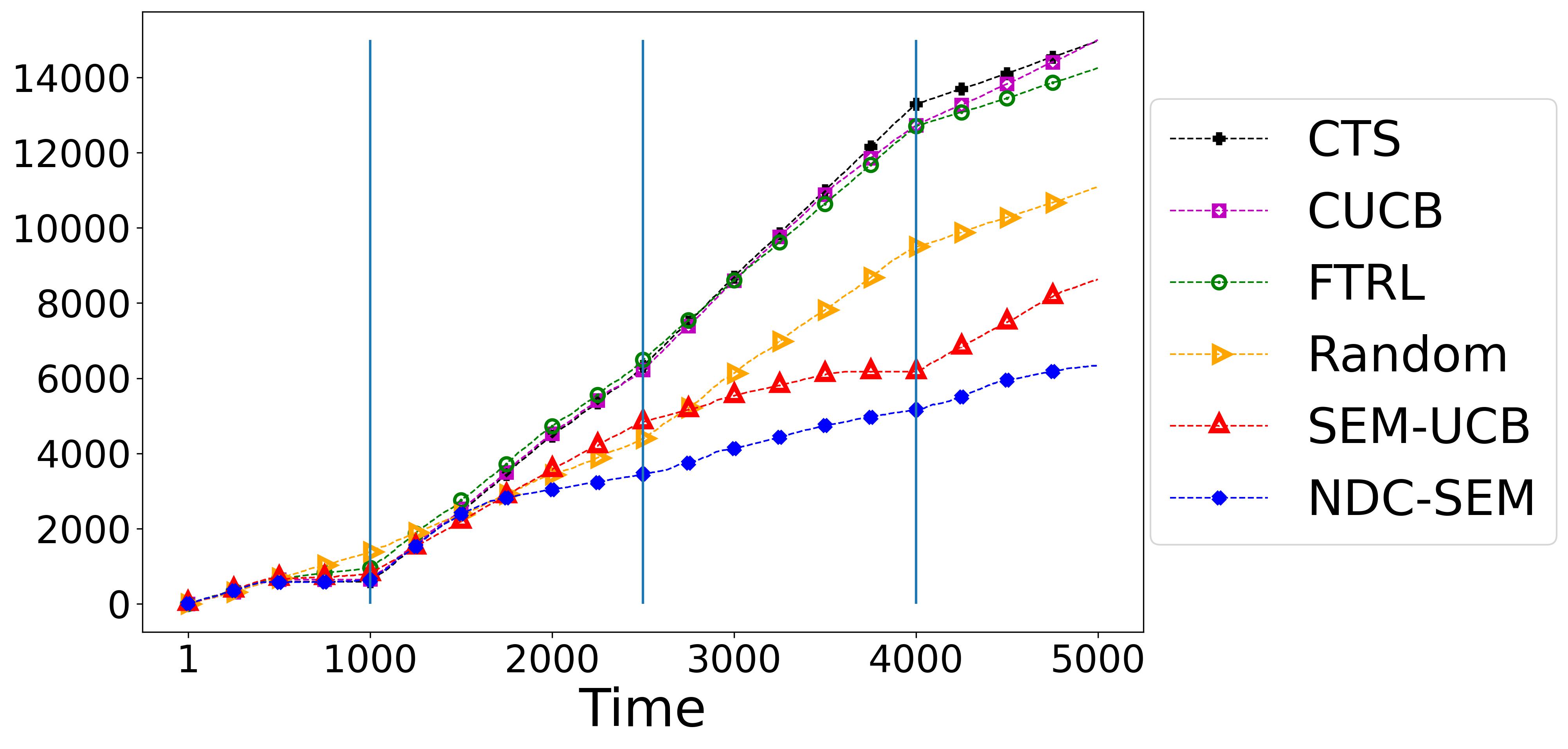}
         \label{subfig:delay3}
     \end{subfigure}
    \caption{Cumulative expected regret of different policies with delay $D \in \{50, 200, 400\}$ from left to right. Vertical lines show the change points.}
    \label{fig:Reg-Synthetic}
\end{figure*}

We present additional results of our experiments using synthetic data in Appendix \ref{App:AdaptationtoChanges}.

\subsection{Covid-19 Dataset}
\label{subsec:RealData}
In addition to the experiments using synthetic data, we evaluate our proposed algorithm on the Covid-19 outbreak dataset of Italy, which includes the daily new infected cases during the pandemic for different regions.\footnote{https://github.com/pcm-dpc/COVID-19} The NDC bandit formulation provides a suitable framework for analysis of Covid-19 spread for the following reasons: (i) Due to movement between regions, there exists a causal impact amongst the daily new cases of different regions. Hence, in each region, the daily new cases result from the causal spread of Covid-19 amongst the regions \cite{Mastakouri20:CAC} and the region-specific characteristics \cite{guaitoli2021covid}, such as social, cultural, and geographical characteristics. (ii) Each region has a specific exposure risk of Covid-19 infection due to different regional characteristics. Naturally, such exposure risk varies over time as our behavior changes, e.g., due to the start of holiday seasons, quarantine orders, or even temperature variations \cite{Meyer20:ETH, Steiger21:CGA}, or as immunity develops, e.g., due to vaccination coverage. Thus, we are dealing with a changing environment. (iii) Finally, the virus testing results are typically reported or even recorded with a delay. Hence, the daily new cases are associated with a delay.

During the Covid-19 pandemic, containing the virus outbreak has been one of the major concerns of governments. To this end, health authorities have considered different measurements for monitoring the outbreak and detecting the regions likely to become coronavirus hotspots. The examples include the daily number of infected cases, incidence rate, and reproduction number (also known as R-value). For example, Germany monitors the $7$-day incidence rate that shows the number of new infections within the past week per $100,000$ population. Consequently, based on the incidence rate of new infections, the German authorities decide whether to impose restrictions, such as enforcing mask-wearing, implementing curfews, making home office obligatory, and banning travel. However, none of such measurements mentioned above considers a region's daily cases' impact on other regions' daily cases. Thus, it is only natural that health authorities seek to find the regions that contribute the most to the total number of daily new cases in the country \cite{bridgwater2021identifying}. By the end of this experiment, we address this critical problem and highlight that our algorithm can detect the optimal candidate regions for political interventions. To our knowledge, no previous work simultaneously considers delay, non-stationarity, and casual impacts amongst regions when analyzing the spread of a contagious disease such as Covid-19.

In the following, we follow our terminology in Section \ref{sec:ProFor} and use the \textit{overall reward} $\mathbf{y}_{t}[i]$ and the \textit{instantaneous reward} $\mathbf{b}_{t}[i]$ to refer to the \textit{overall daily new cases} and the \textit{region-specific daily new cases} in region $i$ at each time (day) $t$, respectively. Naturally, the overall daily new cases include the region-specific daily new cases. 

\textbf{Settings and Data Preparation:}
We consider a period with $T = 80$ days that corresponds to recorded daily new cases from $31$ July to $18$ October, $2020$, for $N = 21$ regions within Italy. The dataset includes only the region's overall daily new cases. Thus, to apply our algorithm, we estimate the underlying distributions of the region-specific daily new cases using a kernel density estimation (See Appendix \ref{App:Region-SpecificDist} for detailed information). We sample from the aforementioned estimated distributions to create the region-specific daily cases for each region. Afterward, to simulate piece-wise stationary reward generating processes, we consider $\Upsilon_{T} = 1$ change point at the day $t= 40$. At the change point, we draw a random integer $k \in \{1, \dots, N-1\}$ and shift the base arms cyclically $k$ times forward. Hence, the instantaneous and overall reward of region $i$ becomes those of region $(i + k - 1 ~\textup{mod}~ N) + 1$. This guarantees that the expected instantaneous reward is piece-wise constant with respect to time. In our experiment, we choose $s = 5$ and consider a delay of $3$ days in receiving the testing results. 
Finally, we tune the parameters of NDC-SEM by performing a grid search and set them to $\gamma = 0.85$ and $\xi = 0.1$. In Appendix \ref{App:HyperparameterTuning}, we elaborate more on the tuning process of parameters.


\begin{figure}[t!]
    \centering
    \includegraphics[width=0.7\textwidth]{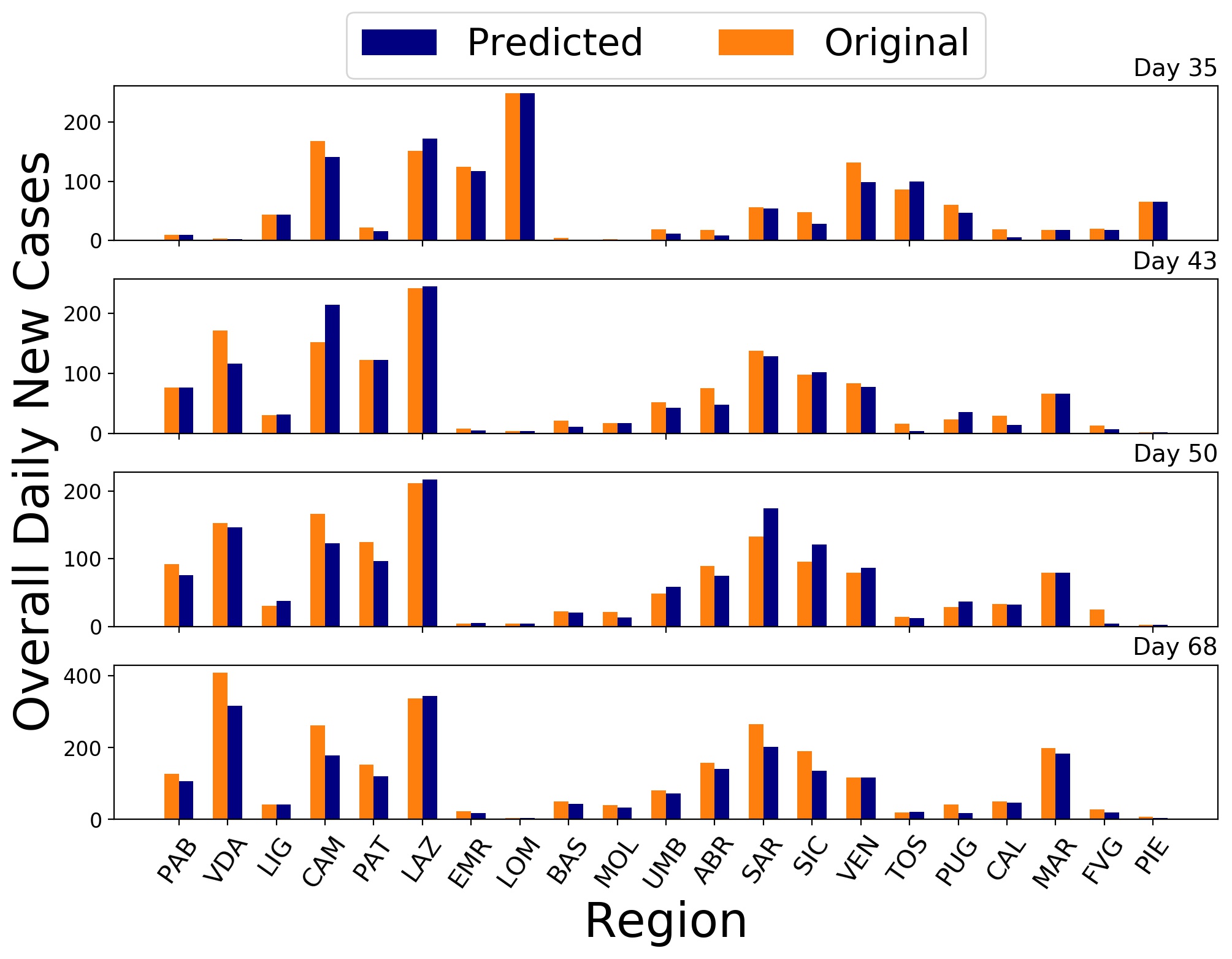}
    \caption{Comparison of the original overall daily new cases and the corresponding predicted values for different days in the validation set.}
    \label{fig:Error}
\end{figure}
\textbf{Learning the Causal Relationships under Delayed Feedback:}
The first learning component in our proposed policy corresponds to learning the ground truth adjacency matrix $\mathbf{A}$ using (\ref{eq:Optimization-A}). To be more realistic, since the causal spread of the disease might create cycles, we include cyclic graphs in the search space of the optimization problem (\ref{eq:Optimization-A}). Further, we split the data into train and validation (tuning) sets in a 90:10 ratio with $72$ and $8$ data samples, respectively. More specifically, we consider $8$ subsets of consecutive days, each with a length of $10$ days. We pick one day in each subset to include in the validation set and add the remaining $9$ days to the train set. The validation set is then used to tune the regularization parameter $\lambda$ online, i.e., by using the already collected validation data up to the current time. At day $t$, we calculate the prediction error as $\epsilon(t) = \frac{1}{N K(t)} \sum_{\tau \in \mathcal{K}(t)} \left\| \mathbf{y}_{\tau} - \hat{\mathbf{y}}_{\tau} \right\|_{1}$,
%
where $\mathcal{K}(t)$ is the validation set at day $t$ with cardinality $K(t) = |\mathcal{K}(t)|$. Moreover, $\mathbf{y}_{\tau}$ and $\hat{\mathbf{y}}_{\tau}$ are the ground truth validation data and the corresponding predicted value using the estimated graph for the day $\tau$, respectively.

\textbf{Fig. \ref{fig:Error}} compares the ground truth overall daily new cases and the corresponding predicted value using the estimated graph on $4$ different days in the validation set. Due to space limitations, we use abbreviations for region names. \textbf{Table \ref{tab:AbbList}} in Appendix \ref{App:NameAbbreviations} lists the original regions' names together with the corresponding abbreviations. As we see, NDC-SEM efficiently estimates the regions' overall daily new cases using the delayed feedback, which helps to improve the decision-making process.

\begin{figure}[t!]
    \centering
    \includegraphics[width=0.75\textwidth]{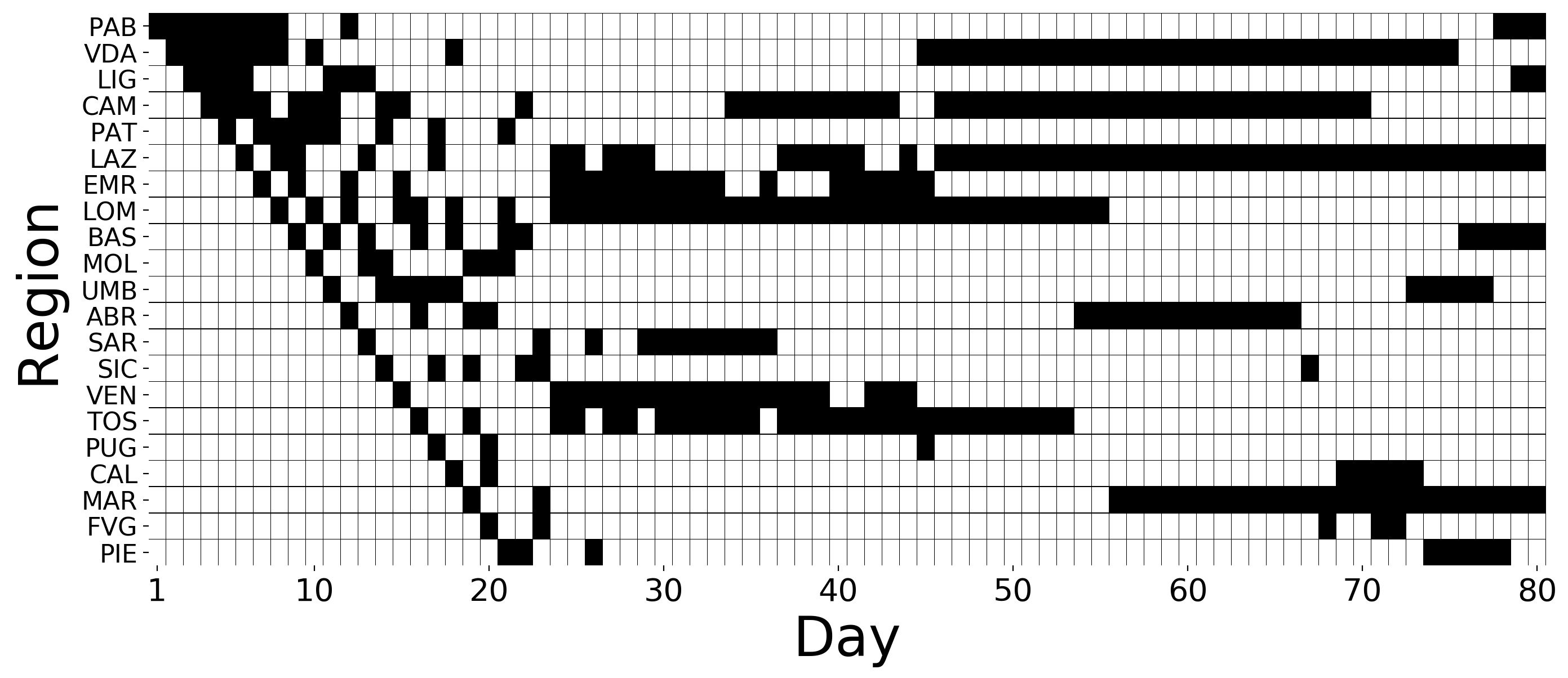}
    \caption{Selected regions by NDC-SEM on each day.}
    \label{fig:SuperArmSelection}
\end{figure}
\textbf{Adaptive Learning of the Regions with Highest Contribution:}
Using the setup mentioned above, we run the NDC-SEM algorithm and show the agent's decision-making over time in \textbf{Fig. \ref{fig:SuperArmSelection}}. The $5$ selected regions at each day are shown by black rectangles. Based on our framework, we represent the selected regions as those with the highest contributions to the Covid-19 spread during the study period of our experiment. As we see, NDC-SEM adaptively selects the regions over time; that is why some selected regions after the change point differ from those selected before the change point. For example, Marche, Abruzzo, and Valle d’Aosta regions are selected only after the change point.

The above-explained adaptive selection of regions is a significant advantage over the SEM-UCB benchmark policy, as SEM-UCB does not consider the non-stationarity and the delay. This is also evident from \textbf{Fig. \ref{fig:EstCollectedOverallRew}} in Appendix \ref{App:comparison}, where we show that NDC-SEM achieves a higher estimated cumulative overall reward compared to SEM-UCB. Notably, each region's contribution to the Covid-19 development differs from its overall daily cases of infection due to the existing causal effects amongst the regions. Therefore, the set of regions with the highest contributions is not necessarily the same as the set of regions with the highest total number of daily cases. In addition, in a real-world scenario, the set of regions with the highest contributions might change over time in a non-stationary environment. This is a key aspect of our problem formulation, which NDC-SEM addresses in Fig. \ref{fig:SuperArmSelection}. 

\section{Conclusion} 
\label{sec:Conclusion}
%
    In this paper, we introduced the NDC bandit framework that addresses real-world problems where the feedback is delayed, the environment is non-stationary, and the base arm's rewards are causally related. We developed a decision-making policy, namely NDC-SEM, that learns the causal relationships using the delayed feedback and alleviates the effects of changes in non-stationary environments by discounting distant past rewards. We analyzed NDC-SEM theoretically and showed that it outperforms several state-of-the-art bandit algorithms. 
    %
    We employed our proposed framework to detect the regions that contribute the most to the spread of Covid-19 within Italy. 
    %
    
    Beside the Covid-19 problem, our method can be applied to analyze gene regulatory networks, financial networks, or even artificial neural networks in online settings.
    %
    The first future research direction would be to extend the current framework by considering confounding variables. Another potential extension of our work would be to consider the contextual version of the NDC bandit problem, where the rewards of each base arm depend on a given context vector.
\section{Appendix}
\label{sec:App}
%
%
\subsection{Notations}
\label{sec:Notation}
Before proceeding to the proof, in the following we introduce some important notations together with their definitions.

For any positive $T$, we define $\Gamma(\gamma)$ as 
\begin{align}
\Gamma(\gamma) = {\Big\{} t \in \{N+1, \dots, T\} ~{\Big|}~ 
\boldsymbol{\beta}_{s}[j] = \boldsymbol{\beta}_{t}[j], \forall j \in [N], \forall s \hspace{1mm} \text{s.t.} \hspace{1mm} t - J(\gamma) < s \leq t {\Big\}},
\end{align}
%
where
\begin{align}
    J(\gamma) = \frac{ \log{( (1-\gamma) \xi (s+1) \log{m_{N}^{\gamma}} )} }{ \log{\gamma} }.
\end{align}
%

Let $\mathcal{I}(\mathbf{x}) = \left\{ i \in [N] ~|~ \mathbf{x}[i] \neq 0 \right\}$ denote the \textit{index set} for a decision vector $\mathbf{x} \in \mathcal{X}$. For each base arm $i$ at time $t$, we define $\mathbf{C}_{t}[i] = 2 \sqrt{\frac{\xi (s+1) \log{m_{t}^{\gamma}}}{\mathbf{M}_{t}^{\gamma,D}[i]}}$. At each time $t$, we collect the computed values of $\hat{\boldsymbol{\beta}}_{t}[i]$ and $\mathbf{C}_{t}[i]$ for all base arms $i \in [N]$ in vectors $\hat{\boldsymbol{\beta}}_{t}$ and $\mathbf{C}_{t}$, respectively. Therefore, based on the definition of UCB indices in (13), we have $\mathbf{E}_{t} = \hat{\boldsymbol{\beta}}_{t} + \mathbf{C}_{t}$. For ease of presentation, in the sequel, we use the following equivalence ${\bf 1}^{\top} (\mathbf{I} - \hat{\mathbf{A}}_{t-1})^{-1} \text{diag}(\mathbf{E}_{t-1}) \mathbf{x}_{t} = {\bf 1}^{\top} (\mathbf{I} - \hat{\mathbf{A}}_{t-1})^{-1} \text{diag}(\mathbf{x}_{t}) \mathbf{E}_{t-1}$. At each time $t$, we define the \textit{selection index} for a decision vector $\mathbf{x} \in \mathcal{X}$ as $I_{t}(\mathbf{x}) = \mathbf{1}^{\top}({\bf I} - \hat{\mathbf{A}}_{t-1})^{-1} \textup{diag}(\mathbf{x}) \mathbf{E}_{t-1}$. To simplify the notation, sometimes we drop the time index $t$ in $\mathbf{M}_{t}^{\gamma,D}[i]$ and use $\mathbf{M}^{\gamma,D}[i]$ to denote the discounted number of times that the base arm $i$ has been observed up to the current time instance minus delay.

For each base arm $i \in [N]$, we define a counter $\mathscr{T}_{i}(t)$ which is updated as follows. At each time $t$ that a suboptimal decision vector $\mathbf{x}_{t}$ is selected, we have at least one base arm $i \in [N]$ such that $i = \underset{i \in \mathcal{I}(\mathbf{x}_{t})}{\textup{argmin}} ~\mathbf{M}_{t-1}^{\gamma,D}[i]$. In this case, if the base arm $i$ is unique, we increment $\mathscr{T}_{i}(t)$ by 1. If there is more than one such base arm, we break the tie and select one of them arbitrarily to increment its corresponding counter. Finally, by $\mathbbm{I}_{i}(t)$, we denote the indicator function which is equal to $1$ if $\mathscr{T}_{i}(t)$ is increased by $1$ at time $t$, and is $0$ otherwise.
%
\subsection{Main Results}
\label{app:mainresults}
We use the following lemma in the proof of Theorem \ref{thm:1}. 
\begin{lemma}{}
\label{lem:CounterBound}
For any $i \in [N]$ and any integers $W, D > 0$, let $\mathbf{M}_{t-W:t-D}[i] = \sum\limits_{\tau = t-W+1}^{t-D} \hspace{-1.1mm} \mathbbm{1} \left\{ \mathbbm{I}_{i}(\tau) = 1\right\}$, where $\mathbbm{I}_{i}(t)$ is the indicator function defined above. Then, for any $\ell > 0$,
\begin{align}
\label{eq:lemma1hyp}
    \sum_{t = N+1}^{T} \mathbbm{1} \left\{ \mathbbm{I}_{i}(t) = 1~\&~\mathbf{M}_{t-1}^{\gamma,D}[i] < \ell \right\} \leq \lceil{ \frac{T}{W} \rceil} (\ell \gamma^{-W} + D). 
\end{align}
\end{lemma}
\begin{proof}
First, we prove that
\begin{align}
\label{eq:ProofforUndiscountedSum}
    \sum_{t = N+1}^{T} \mathbbm{1} \left\{ \mathbbm{I}_{i}(t) = 1~\&~\mathbf{M}_{t-W:t-D}[i] < \ell \right\} \leq \lceil{ \frac{T}{W} \rceil} (\ell + D). 
\end{align}
We have
\begin{align} \nonumber
    \sum_{t = N+1}^{T} 
    \mathbbm{1} &\left\{ \mathbbm{I}_{i}(t) = 1~\&~\mathbf{M}_{t-W:t-D}[i] < \ell \right\} \\
    &\leq \sum_{\tau = 1}^{\lceil{ T/W \rceil}} \sum_{t = (\tau - 1)W + 1}^{\tau W} \mathbbm{1} \left\{ \mathbbm{I}_{i}(t)~\&~\mathbf{M}_{t-W:t-D}[i] < \ell \right\}.
\end{align}
%
For any $\tau \in \{1, \dots, \lceil{ \frac{T}{W} \rceil}\}$, either $\sum\limits_{t = (\tau - 1)W + 1}^{\tau W} \mathbbm{1} \left\{ \mathbbm{I}_{i}(t) = 1~\&~\mathbf{M}_{t-W:t-D}[i] < \ell \right\} = 0$, or there exists a time point $t \in \{(\tau - 1)W + 1, \dots, \tau W\}$ such that $\mathbbm{I}_{i}(t) = 1$ and $\mathbf{M}_{t-W:t-D}[i] < \ell$. In such case, let $t_{\tau} = \max\{ t \in \{(\tau - 1)W + 1, \dots, \tau W\}~|~\mathbbm{I}_{i}(t) = 1~\&~\mathbf{M}_{t-W:t-D}[i] < \ell\}$. Therefore,
\begin{align} \nonumber
\label{eq:lemma-middlestep}
    \sum_{t = (\tau - 1)W + 1}^{\tau W} &\mathbbm{1} \left\{ \mathbbm{I}_{i}(t) = 1~\&~\mathbf{M}_{t-W:t-D}[i] < \ell \right\} \\ \nonumber
    &= \sum_{t = (\tau - 1)W + 1}^{t_{\tau}} \mathbbm{1} \left\{ \mathbbm{I}_{i}(t) = 1~\&~\mathbf{M}_{t-W:t-D}[i] < \ell \right\} \\ \nonumber
    &\leq \sum_{t = t_{\tau} - W + 1}^{t_{\tau}} \mathbbm{1} \left\{ \mathbbm{I}_{i}(t) = 1~\&~\mathbf{M}_{t-W:t-D}[i] < \ell \right\} \\ 
    &\leq \sum_{t = t_{\tau} - W + 1}^{t_{\tau}} \mathbbm{1} \left\{ \mathbbm{I}_{i}(t) = 1 \right\}
    \leq \mathbf{M}_{t_{\tau}-W:t_{\tau}-D}[i] + D
    < \ell + D.
\end{align}
Therefore, we prove (\ref{eq:ProofforUndiscountedSum}). We conclude the proof of lemma using the following observation.
\begin{align} \nonumber
\label{eq:lemma-laststep}
    \sum_{t = N+1}^{T} 
    &\mathbbm{1} \left\{ \mathbbm{I}_{i}(t) = 1~\&~\mathbf{M}_{t-1}^{\gamma,D}[i] < \ell \right\} \\
    &\leq \sum_{t = N+1}^{T} \mathbbm{1} \left\{ \mathbbm{I}_{i}(t) = 1~\&~\mathbf{M}_{t-W:t-D}[i] < \ell \gamma^{-W} \right\}.
\end{align}
\end{proof}
\begin{corollary}
\label{cor:RandomDelay}
In the specific case where the delay is a random variable given by $D_{t} \leq D_{\max}$, for any $t \in [T]$, and $D_{\max} = \max_{t \in [T]} D_{t}$, Lemma \ref{lem:CounterBound} can be rewritten as
\begin{align}
\label{eq:cor1hyp}
    \sum_{t = N+1}^{T} \mathbbm{1} \left\{ \mathbbm{I}_{i}(t) = 1~\&~\mathbf{M}_{t-1}^{\gamma,D_{t}}[i] < \ell \right\} \leq  \lceil{ \frac{T}{W} \rceil} (\ell \gamma^{-W} + D_{\max}). 
\end{align}
\begin{proof}
When delay is random, we have $\mathbf{M}_{t}^{\gamma,D_{t}}[i] = \sum_{\tau = 1}^{t-D_{t}} \gamma^{t - \tau} \mathbbm{1}\left\{ \mathbf{x}_{\tau}[i] = 1\right\}$ and $\mathbf{M}_{t-W:t-D_{t}}[i] = \sum_{\tau = t-W+1}^{t-D_{t}} \mathbbm{1} \left\{ \mathbbm{I}_{i}(\tau) = 1\right\}$. In addition, for any $i \in [N]$, and any $t \in [T]$, we have $\mathbf{M}_{t-W:t-D_{\max}}[i] \leq \mathbf{M}_{t-W:t-D_{t}}[i]$. Therefore, we can rewrite (\ref{eq:lemma-middlestep}) as
\begin{align} \nonumber
    \sum_{t = (\tau - 1)W + 1}^{\tau W} &\mathbbm{1} \left\{ \mathbbm{I}_{i}(t) = 1~\&~\mathbf{M}_{t-W:t-D}[i] < \ell \right\} \\ \nonumber
    &\leq \sum_{t = t_{\tau} - W + 1}^{t_{\tau}} \mathbbm{1} \left\{ \mathbbm{I}_{i}(t) = 1~\&~\mathbf{M}_{t-W:t-D_{t}}[i] < \ell \right\} \\ \nonumber
    &\leq \sum_{t = t_{\tau} - W + 1}^{t_{\tau}} \mathbbm{1} \left\{ \mathbbm{I}_{i}(t) = 1~\&~\mathbf{M}_{t-W:t-D_{\max}}[i] < \ell \right\} \\ 
    &\leq \sum_{t = t_{\tau} - W + 1}^{t_{\tau}} \mathbbm{1} \left\{ \mathbbm{I}_{i}(t) = 1 \right\}
    \leq \mathbf{M}_{t_{\tau}-W:t_{\tau}-D_{\max}}[i] + D_{\max} 
    < \ell + D_{max}.
\end{align}
We conclude the proof by observing that (\ref{eq:lemma-laststep}) holds for $D = D_{\max}$.
\end{proof}
\end{corollary}
\subsubsection{Proof of Theorem \ref{thm:1}} 
\label{subsec:ProofThm1}
\begin{proof}
%
%
%
We rewrite the expected regret as
\begin{align}
    \mathcal{R}_{T}(\mathcal{X}) 
    = \sum_{t = 1}^{T} \left[\mu_{t}(\mathbf{x}_{t}^{\ast}) - \mu_{t}(\mathbf{x}_{t})\right]
    = \mathbb{E} \left[\sum_{t=1}^{T}
    \Delta_{t}(\mathbf{x}_{t}) \mathbbm{1}\{\mathbf{x}_{t} \neq \mathbf{x}_{t}^{\ast}\}\right]  \stackrel{(\ast)}{\leq} \Delta_{\max} \mathbb{E} \left[\sum_{t=1}^{T} \mathbbm{1}\{\mathbf{x}_{t} \neq \mathbf{x}_{t}^{\ast}\}\right],
\end{align}
where $(\ast)$ follows from the definition of $\Delta_{\max}$. 

Based on the definition of the counters $\mathscr{T}_{i}(t)$ for the base arms $i \in [N]$, at each time $t$ that a suboptimal decision vector is selected, only one of such counters is incremented by $1$. Thus, we have \cite{Gai12:CNO}
\begin{align}
    \mathbb{E} \left[ \sum_{t=1}^{T} \mathbbm{1}\{\mathbf{x}_{t} \neq \mathbf{x}_{t}^{\ast}\} \right]
    = \mathbb{E} \left[\sum_{i = 1}^{N} \mathscr{T}_{i}(t)\right] = \sum_{i = 1}^{N} \mathbb{E} \left[\mathscr{T}_{i}(t)\right].
\end{align}
%
%
Therefore, we observe that
\begin{align}
    \mathcal{R}_{T}(\mathcal{X})
    \leq \Delta_{\max} \mathbb{E} \left[\sum_{t=1}^{T} \mathbbm{1}\{\mathbf{x}_{t} \neq \mathbf{x}_{t}^{\ast}\}\right] = \Delta_{\max} \sum_{i = 1}^{N} \mathbb{E} [\mathscr{T}_{i}(T)].
\end{align}

Recall that $\mathbbm{I}_{i}(t)$ is the indicator function which is equal to $1$ if $\mathscr{T}_{i}(t)$ is increased by $1$ at time $t$, and is $0$ otherwise. Hence,
\begin{align}
    \mathscr{T}_{i}(T) = \sum_{t = N+1}^{T} \mathbbm{1}\left\{ \mathbbm{I}_{i}(t) = 1 \right\}.
\end{align}
If $\mathbbm{I}_{i}(t) = 1$, it means that a suboptimal decision vector $\mathbf{x}_{t}$ is selected at time $t$. In this case, $\mathbf{M}_{t-1}^{\gamma,D}[i] = \min \left\{ \mathbf{M}_{t-1}^{\gamma,D}[j] | j \in \mathcal{I}(\mathbf{x}_{t}) \right\}$. Let $\ell = \left \lceil {\frac{16 \xi (s+1) \log{m_{T}^{\gamma}}}{ (\frac{\Delta_{\min}}{s w_{\max}})^{2} }} \right \rceil$. Then,
%
\begin{align} \nonumber
   \mathscr{T}_{i}(T) 
   &= \sum_{t = N+1}^{T} \mathbbm{1}\left\{ \mathbbm{I}_{i}(t) = 1 \right\} \\ \nonumber
   &\leq 1 + \sum_{t = N+1}^{T} \mathbbm{1}\left\{ \mathbbm{I}_{i}(t) = 1 ~\&~ \mathbf{M}_{t-1}^{\gamma,D}[i] < \ell \right\} 
   + \sum_{t = N+1}^{T} \mathbbm{1}\left\{ \mathbbm{I}_{i}(t) = 1 ~\&~ \mathbf{M}_{t-1}^{\gamma,D}[i] \geq \ell \right\}   \\ \nonumber
   &\stackrel{(\ast)}{\leq} 1 +  \lceil{T(1-\gamma)\rceil}(\ell \gamma^{-\frac{1}{1-\gamma}} + D) + J(\gamma) \Upsilon_{T} 
   + \sum_{t \in \Gamma(\gamma)}^{} \mathbbm{1}\left\{ I_{t}(\mathbf{x}_{t}^{\ast}) \leq I_{t}(\mathbf{x}_{t}) ~\&~ \mathbf{M}_{t-1}^{\gamma,D}[i] \geq \ell \right \} \\ \nonumber 
   &= 1 + J(\gamma) \Upsilon_{T} +  \lceil{T(1-\gamma)\rceil}(\ell \gamma^{-\frac{1}{1-\gamma}} + D) \\ \nonumber
   &\hspace{16mm}+\sum_{t \in \Gamma(\gamma)}^{} \mathbbm{1} {\Big\{} {\bf 1}^{\top} (\mathbf{I} - \hat{\mathbf{A}}_{t - 1})^{-1} \text{diag}(\mathbf{x}_{t}^{\ast}) \mathbf{E}_{t-1} \\
   &\hspace{50mm}\leq {\bf 1}^{\top} (\mathbf{I} - \hat{\mathbf{A}}_{t - 1})^{-1} \text{diag}(\mathbf{x}_{t}) \mathbf{E}_{t-1} 
   ~\&~ \mathbf{M}_{t-1}^{\gamma,D}[i] \geq \ell {\Big\}},
\end{align}
where $(\ast)$ follows from Lemma \ref{lem:CounterBound} by choosing $W = \frac{1}{1-\gamma}$.

Note that, when $\mathscr{T}_{i}(t) $ is incremented by $1$ at time $t$ and $\mathbf{M}_{t-1}^{\gamma,D}[i] \geq \ell$, the following holds.
\begin{align}
    \ell \leq \mathbf{M}_{t-1}^{\gamma,D}[i] \leq \mathbf{M}_{t-1}^{\gamma,D}[j], ~~~~ \forall j \in \mathcal{I}(\mathbf{x}_{t}).
\end{align}
Let $\mathbf{v}_{t}^{\top} = {\bf 1}^{\top} (\mathbf{I} - \hat{\mathbf{A}}_{t-1})^{-1} \text{diag}(\mathbf{x}_{t}^{\ast})$ and $\mathbf{u}_{t}^{\top} = {\bf 1}^{\top} (\mathbf{I} - \hat{\mathbf{A}}_{t-1})^{-1} \text{diag}(\mathbf{x}_{t})$. We order the elements in sets $\mathcal{I}(\mathbf{x}_{t}^{\ast})$ and $\mathcal{I}(\mathbf{x}_{t})$ arbitrarily. In the following, our results are independent of the way we order these sets. Let $v_{k}$, $k = 1, \dots, |\mathcal{I}(\mathbf{x}_{t}^{\ast})| \leq s$, represent the $k$th element in $\mathcal{I}(\mathbf{x}_{t}^{\ast})$ and $u_{k}$, $k = 1, \dots, |\mathcal{I}(\mathbf{x}_{t})| \leq s$, represent the $k$th element in $\mathcal{I}(\mathbf{x}_{t})$. Hence, we have
\begin{align} \nonumber
\label{eq:event}
   \mathscr{T}_{i}(T)
   &\leq 1 + J(\gamma) \Upsilon_{T} +  \lceil{T(1-\gamma)\rceil}(\ell \gamma^{-\frac{1}{1-\gamma}} + D) \\ \nonumber
   &+\sum_{t \in \Gamma(\gamma)}^{} \mathbbm{1} {\Bigg\{} \min_{0 < \mathbf{M}^{\gamma,D}[v_{1}], \dots, \mathbf{M}^{\gamma,D}[v_{|\mathcal{I}(\mathbf{x}_{t}^{\ast})|}] \leq t} 
   \sum_{j=1}^{|\mathcal{I}(\mathbf{x}_{t}^{\ast})|} \mathbf{v}_{t}^{\top}[v_{j}] (\hat{\boldsymbol{\beta}}_{t-1}[v_{j}] + \mathbf{C}_{t-1}[v_{j}]) \leq \\ \nonumber
   &\hspace{35mm}\max_{\ell \leq \mathbf{M}^{\gamma,D}[u_{1}], \dots,
   \mathbf{M}^{\gamma,D}[u_{|\mathcal{I}(\mathbf{x}_{t})|}] \leq t}
   \sum_{j=1}^{|\mathcal{I}(\mathbf{x}_{t})|} \mathbf{u}_{t}^{\top}[u_{j}] (\hat{\boldsymbol{\beta}}_{t-1}[u_{j}] + \mathbf{C}_{t-1}[u_{j}]) {\Bigg \}} \\ \nonumber
   &\leq 1 + J(\gamma) \Upsilon_{T} +  \lceil{T(1-\gamma)\rceil}(\ell \gamma^{-\frac{1}{1-\gamma}} + D) \\ \nonumber
   &\hspace{5mm}+\sum_{t \in \Gamma(\gamma)}^{}  \sum_{\mathbf{M}^{\gamma,D}[v_{1}]=1}^{t} \dots \sum_{\mathbf{M}^{\gamma,D}[v_{|\mathcal{I}(\mathbf{x}_{t}^{\ast})|}]=1}^{t}
   \sum_{\mathbf{M}^{\gamma,D}[u_{1}]=\ell}^{t} \dots \sum_{\mathbf{M}^{\gamma,D}[u_{|\mathcal{I}(\mathbf{x}_{t})|}]=\ell}^{t} \\ 
   &\hspace{15mm}\mathbbm{1} {\Bigg \{} 
   \sum_{j=1}^{|\mathcal{I}(\mathbf{x}_{t}^{\ast})|} \mathbf{v}_{t}^{\top}[v_{j}] (\hat{\boldsymbol{\beta}}_{t-1}[v_{j}] + \mathbf{C}_{t-1}[v_{j}]) 
   \leq \sum_{j=1}^{|\mathcal{I}(\mathbf{x}_{t})|} \mathbf{u}_{t}^{\top}[u_{j}] (\hat{\boldsymbol{\beta}}_{t-1}[u_{j}] + \mathbf{C}_{t-1}[u_{j}])
   {\Bigg \}}. 
\end{align}
We define the Event $\mathcal{P}$ as
\begin{align}
\label{eq:eventp-separated}
    \sum_{j=1}^{|\mathcal{I}(\mathbf{x}_{t}^{\ast})|} \mathbf{v}_{t}^{\top}&[v_{j}] (\hat{\boldsymbol{\beta}}_{t-1}[v_{j}] + \mathbf{C}_{t-1}[v_{j}])
   \leq \sum_{j=1}^{|\mathcal{I}(\mathbf{x}_{t})|} \mathbf{u}_{t}^{\top}[u_{j}] (\hat{\boldsymbol{\beta}}_{t-1}[u_{j}] + \mathbf{C}_{t-1}[u_{j}]).
\end{align}
Now, for $t \in \Gamma(\gamma)$, if the Event $\mathcal{P}$ in (\ref{eq:eventp-separated}) is true, it implies that at least one of the following events must be true.
\begin{align} \label{eq:part1} 
    &{\bf 1}^{\top} (\mathbf{I} - \hat{\mathbf{A}}_{t-1})^{-1} \text{diag}(\mathbf{x}_{t}^{\ast}) (\hat{\boldsymbol{\beta}}_{t-1} + \mathbf{C}_{t-1})
    \leq {\bf 1}^{\top} (\mathbf{I} - \mathbf{A})^{-1} \text{diag}(\mathbf{x}_{t}^{\ast})\boldsymbol{\beta}_{t-1}, \\
    \label{eq:part2}
    &{\bf 1}^{\top} (\mathbf{I} - \hat{\mathbf{A}}_{t-1})^{-1} \text{diag}(\mathbf{x}_{t}) (\hat{\boldsymbol{\beta}}_{t-1} - \mathbf{C}_{t-1})
    \geq {\bf 1}^{\top} (\mathbf{I} - \mathbf{A})^{-1} \text{diag}(\mathbf{x}_{t})\boldsymbol{\beta}_{t-1}, \\ 
    \label{eq:part3}
    &{\bf 1}^{\top} (\mathbf{I} - \mathbf{A})^{-1} \text{diag}(\mathbf{x}_{t}^{\ast})\boldsymbol{\beta}_{t-1}
    < {\bf 1}^{\top} (\mathbf{I} - \mathbf{A})^{-1} \text{diag}(\mathbf{x}_{t})\boldsymbol{\beta}_{t-1} + 2 {\bf 1}^{\top} (\mathbf{I} - \hat{\mathbf{A}}_{t-1})^{-1} \text{diag}(\mathbf{x}_{t}) \mathbf{C}_{t-1}.
    \end{align}
First, we consider (\ref{eq:part1}). Based on our problem formulation and proposed solution, we know that matrices $\mathbf{A}$ and $\hat{\mathbf{A}}_{t-1}$ are nilpotent with index $N$. Thus, $\mathbf{A}^{N} = \mathbf{0}_{N \times N}$ and $\hat{\mathbf{A}}_{t-1}^{N} = \mathbf{0}_{N \times N}$. Hence, we can write the Taylor's series of $(\mathbf{I} - \mathbf{A})^{-1}$ and $(\mathbf{I} - \hat{\mathbf{A}}_{t-1})^{-1}$ as
\begin{align}
\label{eq:nilpotentA}
    \hspace{-7mm}(\mathbf{I} - \mathbf{A})^{-1} = \mathbf{I} + \mathbf{A} + \mathbf{A}^{2} + \dots + \mathbf{A}^{N-1},
\end{align}
and
\begin{align}
\label{eq:nilpotentAhat}
    (\mathbf{I} - \hat{\mathbf{A}}_{t-1})^{-1} = \mathbf{I} + \hat{\mathbf{A}}_{t-1} + \hat{\mathbf{A}}_{t-1}^{2} + \dots + \hat{\mathbf{A}}_{t-1}^{N-1},
\end{align}
respectively. Substituting (\ref{eq:nilpotentA}) and (\ref{eq:nilpotentAhat}) in (\ref{eq:part1}) results in
\begin{align} 
    {\bf 1}^{\top} (\mathbf{I} + \hat{\mathbf{A}}_{t-1} + \dots + \hat{\mathbf{A}}_{t-1}^{N-1}) &\text{diag}(\mathbf{x}_{t}^{\ast}) (\hat{\boldsymbol{\beta}}_{t-1} + \mathbf{C}_{t-1}) 
    \leq {\bf 1}^{\top} (\mathbf{I} + \mathbf{A} + \dots + \mathbf{A}^{N-1}) \text{diag}(\mathbf{x}_{t}^{\ast}) \boldsymbol{\beta}_{t-1}.
\end{align}
For $j = 1, \dots N$, we find the upper bound for
\begin{align} 
\label{eq:part1-j}
    \mathbb{P} {\Big[} {\bf 1}^{\top} \hat{\mathbf{A}}_{t-1}^{j-1} \text{diag}(\mathbf{x}_{t}^{\ast}) (\hat{\boldsymbol{\beta}}_{t-1} + \mathbf{C}_{t-1}) 
    \leq
    {\bf 1}^{\top} \mathbf{A}^{j-1} \text{diag}(\mathbf{x}_{t}^{\ast}) \boldsymbol{\beta}_{t-1} {\Big]}.
\end{align}
We consider the following Event $\mathcal{E}$.
%
%
%
\begin{align} \nonumber
    {\bf 1}^{\top} \hat{\mathbf{A}}_{t-1}^{j-1} \text{diag}(\mathbf{x}_{t}^{\ast}) (\hat{\boldsymbol{\beta}}_{t-1} + \mathbf{C}_{t-1}) 
    &+ {\bf 1}^{\top} \hat{\mathbf{A}}_{t-1}^{j-1} \text{diag}(\mathbf{x}_{t}^{\ast}) \boldsymbol{\beta}_{t-1} \\
    &\leq {\bf 1}^{\top} \hat{\mathbf{A}}_{t-1}^{j-1} \text{diag}(\mathbf{x}_{t}^{\ast}) \boldsymbol{\beta}_{t-1} + {\bf 1}^{\top} \mathbf{A}^{j-1} \text{diag}(\mathbf{x}_{t}^{\ast}) \boldsymbol{\beta}_{t-1}.
\end{align}
If $\mathcal{E}$ is true, then at least one of the following must hold.
\begin{align} 
    &\underbrace{{\bf 1}^{\top} \hat{\mathbf{A}}_{t-1}^{j-1} \text{diag}(\mathbf{x}_{t}^{\ast}) (\hat{\boldsymbol{\beta}}_{t-1} + \mathbf{C}_{t-1})
    \leq {\bf 1}^{\top} \hat{\mathbf{A}}_{t-1}^{j-1} \text{diag}(\mathbf{x}_{t}^{\ast})\boldsymbol{\beta}_{t-1}}_{\mathcal{I}}, \\
    &\underbrace{{\bf 1}^{\top} \hat{\mathbf{A}}_{t-1}^{j-1} \text{diag}(\mathbf{x}_{t}^{\ast}) \boldsymbol{\beta}_{t-1}
    \leq {\bf 1}^{\top} \mathbf{A}^{j-1} \text{diag}(\mathbf{x}_{t}^{\ast})\boldsymbol{\beta}_{t-1}}_{\mathcal{II}}.
\end{align}
Therefore, we have
\begin{align} 
\label{eq:eventE}
    \mathbb{P} \left[ \mathcal{E} \right] \leq \mathbb{P} \left[ \mathcal{I} \right] + \mathbb{P} \left[ \mathcal{II} \right].
\end{align}
Let $\mathbf{y}_{t}^{\top} = {\bf 1}^{\top} \hat{\mathbf{A}}_{t-1}^{j-1} \text{diag}(\mathbf{x}_{t}^{\ast})$. If Event $\mathcal{I}$ is true, then at least one of the following must hold.
\begin{align} 
    \mathbf{y}_{t}^{\top}[v_{1}] (\hat{\boldsymbol{\beta}}_{t-1}[v_{1}] + \mathbf{C}_{t-1}[v_{1}])
    &\leq \mathbf{y}_{t}^{\top}[v_{1}] \boldsymbol{\beta}_{t-1}[v_{1}], \\
    \mathbf{y}_{t}^{\top}[v_{2}] (\hat{\boldsymbol{\beta}}_{t-1}[v_{2}] + \mathbf{C}_{t-1}[v_{2}])
    &\leq \mathbf{y}_{t}^{\top}[v_{2}] \boldsymbol{\beta}_{t-1}[v_{2}], \\ \nonumber
    &\vdots \\ 
    \mathbf{y}_{t}^{\top}[v_{|\mathcal{I}(\mathbf{x}_{t}^{\ast})|}] (\hat{\boldsymbol{\beta}}_{t-1}[v_{|\mathcal{I}(\mathbf{x}_{t}^{\ast})|}] + \mathbf{C}_{t-1}[v_{|\mathcal{I}(\mathbf{x}_{t}^{\ast})|}]) 
    &\leq \mathbf{y}_{t}^{\top}[v_{|\mathcal{I}(\mathbf{x}_{t}^{\ast})|}] \boldsymbol{\beta}_{t-1}[v_{|\mathcal{I}(\mathbf{x}_{t}^{\ast})|}].
\end{align}
For $k = 1, \dots, |\mathcal{I}(\mathbf{x}_{t}^{\ast})|$, we have
\begin{align} \nonumber
    \mathbb{P}{\Big[} \mathbf{y}_{t}^{\top}[v_{k}] (\hat{\boldsymbol{\beta}}_{t-1}[v_{k}] + \mathbf{C}_{t-1}[v_{k}])
    \leq \mathbf{y}_{t}^{\top}[v_{k}] \boldsymbol{\beta}_{t-1}[v_{k}]{\Big]} 
    &\stackrel{(a)}{=} \mathbb{P}{\Big[} (\hat{\boldsymbol{\beta}}_{t-1}[v_{k}] + \mathbf{C}_{t-1}[v_{k}]) \leq \boldsymbol{\beta}_{t-1}[v_{k}]{\Big]}  \\ \nonumber
    &\stackrel{(b)}{\leq} \left\lceil \frac{\log m_{t}^{\gamma}}{\log{(1+\eta)}} \right\rceil e^{-\left( 2\xi (s+1) \log{(m_{t}^{\gamma})}\left( 1 - \frac{\eta^{2}}{16} \right) \right)} \\ &\stackrel{(c)}{=} \left\lceil \frac{\log m_{t}^{\gamma}}{\log{(1+\eta)}} \right\rceil (m_{t}^{\gamma})^{-2\xi (s+1) \left( 1 - \frac{\eta^{2}}{16} \right)},
\end{align}
where $(a)$ holds since $\mathbf{y}_{t}^{\top}[v_{k}] \geq 0$, $\forall k$ and $(b)$ follows from a small modification of the proof in \cite{Garivier11:OUC} for all $\eta > 0$. Hence, for Event $\mathcal{I}$, we conclude that
\begin{align} 
    \mathbb{P}\left[\mathcal{I}\right] 
    \leq |\mathcal{I}(\mathbf{x}_{t}^{\ast})| \left\lceil \frac{\log m_{t}^{\gamma}}{\log{(1+\eta)}} \right\rceil (m_{t}^{\gamma})^{-2\xi (s+1) \left( 1 - \frac{\eta^{2}}{16} \right)} 
    \leq s \left\lceil \frac{\log m_{t}^{\gamma}}{\log{(1+\eta)}} \right\rceil (m_{t}^{\gamma})^{-2\xi (s+1) \left( 1 - \frac{\eta^{2}}{16} \right)}.
\end{align}
Now, we consider Event $\mathcal{II}$. Based on Theorem 1 in \cite{Bazerque13:IOS}, we know that we can identify the adjacency matrix $A$ uniquely by $N$ samples gathered during the initialization period of our proposed algorithm. This means that with probability $1$, after the time point $\theta = N + D + 1 < \infty$, $\hat{\mathbf{A}}_{t-1} = \mathbf{A}$ holds for all $t > \theta$. Therefore, for $t > N + D + 1$, Event $\mathcal{II}$ holds with probability $1$.

Combining the aforementioned results with (\ref{eq:eventE}), we find the upper bound for (\ref{eq:part1-j}) as
\begin{align} \nonumber
\label{eq:part1-jj}
    \mathbb{P} {\Big[} {\bf 1}^{\top} \hat{\mathbf{A}}_{t-1}^{j-1} \text{diag}(\mathbf{x}_{t}^{\ast}) (\hat{\boldsymbol{\beta}}_{t-1} + \mathbf{C}_{t-1})  
    \leq
    {\bf 1}^{\top} &\mathbf{A}^{j-1} \text{diag}(\mathbf{x}_{t}^{\ast})\boldsymbol{\beta}_{t-1} {\Big]} \\
    &\leq s \left\lceil \frac{\log m_{t}^{\gamma}}{\log{(1+\eta)}} \right\rceil (m_{t}^{\gamma})^{-2\xi (s+1) \left( 1 - \frac{\eta^{2}}{16} \right)},
\end{align}
for each $j = 1, \dots, N$. Since $\hat{\mathbf{A}}_{t-1} = \mathbf{A}$, $\forall t > N + D + 1$ and the length of the longest path in the graph is $p$, we can rewrite (\ref{eq:nilpotentA}) and (\ref{eq:nilpotentAhat}) as \cite{duncan2004powers}
\begin{align}
\label{eq:nilpotentA-p}
    (\mathbf{I} - \mathbf{A})^{-1} = \mathbf{I} + \mathbf{A} + \mathbf{A}^{2} + \dots + \mathbf{A}^{p},
\end{align}
and
\begin{align}
\label{eq:nilpotentAhat-p}
    (\mathbf{I} - \hat{\mathbf{A}}_{t-1})^{-1} = \mathbf{I} + \hat{\mathbf{A}}_{t-1} + \hat{\mathbf{A}}_{t-1}^{2} + \dots + \hat{\mathbf{A}}_{t-1}^{p},
\end{align}
respectively. Therefore, by using (\ref{eq:nilpotentA-p}) and (\ref{eq:nilpotentAhat-p}) in place of (\ref{eq:nilpotentA}) and (\ref{eq:nilpotentAhat}), and based on (\ref{eq:part1-jj}), the following holds for (\ref{eq:part1}).
\begin{align} \nonumber
\label{eq:part1-done}
    \mathbb{P} {\Big[} {\bf 1}^{\top} (\mathbf{I} - \hat{\mathbf{A}}_{t-1})^{-1} \text{diag}(\mathbf{x}_{t}^{\ast}) (\hat{\boldsymbol{\beta}}_{t-1} + \mathbf{C}_{t-1}) 
    \leq {\bf 1}^{\top} &(\mathbf{I} - \mathbf{A})^{-1} \text{diag}(\mathbf{x}_{t}^{\ast}) \boldsymbol{\beta}_{t-1} {\Big]}  \\
    &\leq s^{p} \left\lceil \frac{\log m_{t}^{\gamma}}{\log{(1+\eta)}} \right\rceil^{p} (m_{t}^{\gamma})^{-2 p \xi (s+1) \left( 1 - \frac{\eta^{2}}{16} \right)}.
\end{align}
For (\ref{eq:part2}), we have similar results as follows.
\begin{align} \nonumber
\label{eq:part2-done}
    \mathbb{P} {\Big[}  {\bf 1}^{\top} (\mathbf{I} - \hat{\mathbf{A}}_{t-1})^{-1} \text{diag}(\mathbf{x}_{t}) (\hat{\boldsymbol{\beta}}_{t-1} - \mathbf{C}_{t-1}) 
    \geq {\bf 1}^{\top} &(\mathbf{I} - \mathbf{A})^{-1} \text{diag}(\mathbf{x}_{t})\boldsymbol{\beta}_{t-1} {\Big]}  \\
    &\leq s^{p} \left\lceil \frac{\log m_{t}^{\gamma}}{\log{(1+\eta)}} \right\rceil^{p} (m_{t}^{\gamma})^{-2 p \xi (s+1) \left( 1 - \frac{\eta^{2}}{16} \right)}.
\end{align}

Finally, we consider (\ref{eq:part3}). We have
\begin{align} \nonumber
\label{eq:part3-done}
    &{\bf 1}^{\top} (\mathbf{I} - \mathbf{A})^{-1} \text{diag}(\mathbf{x}_{t}^{\ast})\boldsymbol{\beta}_{t-1}
    - {\bf 1}^{\top} (\mathbf{I} - \mathbf{A})^{-1} \text{diag}(\mathbf{x}_{t})\boldsymbol{\beta}_{t-1}  
    - 2 {\bf 1}^{\top} (\mathbf{I} - \hat{\mathbf{A}}_{t-1})^{-1} \text{diag}(\mathbf{x}_{t}) \mathbf{C}_{t-1} \\ \nonumber
    &\stackrel{(a)}{=} {\bf 1}^{\top} (\mathbf{I} - \mathbf{A})^{-1} \text{diag}(\mathbf{x}_{t}^{\ast})\boldsymbol{\beta}_{t-1}
    - {\bf 1}^{\top} (\mathbf{I} - \mathbf{A})^{-1} \text{diag}(\mathbf{x}_{t})\boldsymbol{\beta}_{t-1}  
    - 2 \sum_{j: j \in \mathcal{I}(\mathbf{x}_{t})}^{} \mathbf{w}_{t}^{\top}[j] \mathbf{C}_{t-1}[j] \\ \nonumber
    &\stackrel{(b)}{=} {\bf 1}^{\top} (\mathbf{I} - \mathbf{A})^{-1} \text{diag}(\mathbf{x}_{t}^{\ast})\boldsymbol{\beta}_{t-1}
    - {\bf 1}^{\top} (\mathbf{I} - \mathbf{A})^{-1} \text{diag}(\mathbf{x}_{t})\boldsymbol{\beta}_{t-1} \\ \nonumber 
    &\hspace{85mm}- 4 \sum_{j: j \in \mathcal{I}(\mathbf{x}_{t})}^{} \mathbf{w}_{t}^{\top}[j] \sqrt{\frac{\xi (s+1) \log{m_{t-1}^{\gamma}}}{\mathbf{M}_{t-1}^{\gamma,D}[j]}} \\ \nonumber
    &\stackrel{(c)}{\geq} {\bf 1}^{\top} (\mathbf{I} - \mathbf{A})^{-1} \text{diag}(\mathbf{x}_{t}^{\ast})\boldsymbol{\beta}_{t-1}
    - {\bf 1}^{\top} (\mathbf{I} - \mathbf{A})^{-1} \text{diag}(\mathbf{x}_{t})\boldsymbol{\beta}_{t-1}  
    - 4 s w_{\max} \sqrt{\frac{\xi (s+1) \log{m_{T}^{\gamma}}}{\ell}} \\ \nonumber
    &\stackrel{(d)}{\geq} {\bf 1}^{\top} (\mathbf{I} - \mathbf{A})^{-1} \text{diag}(\mathbf{x}_{t}^{\ast})\boldsymbol{\beta}_{t-1} - {\bf 1}^{\top} (\mathbf{I} - \mathbf{A})^{-1} \text{diag}(\mathbf{x}_{t})\boldsymbol{\beta}_{t-1} - \Delta_{\min}
    \\ 
    &\stackrel{(e)}{\geq} {\bf 1}^{\top} (\mathbf{I} - \mathbf{A})^{-1} \text{diag}(\mathbf{x}_{t}^{\ast})\boldsymbol{\beta}_{t-1} - {\bf 1}^{\top} (\mathbf{I} - \mathbf{A})^{-1} \text{diag}(\mathbf{x}_{t})\boldsymbol{\beta}_{t-1} - \Delta_{t-1}(\mathbf{x}_{t}) = 0,
\end{align}
where in $(a)$ and $(c)$ we used the definition of $\mathbf{w}_{t}^{\top}$ and $w_{\max}$, respectively. Moreover, in $(b)$ and $(d)$, we substituted the value for $\mathbf{C}_{t-1}[j]$ and $l$, respectively. $(e)$ follows from the definition of $\Delta_{\min}$. Hence, we conclude that (\ref{eq:part3}) never happens.

Since $\xi > \frac{1}{2(s+1)}$, we can choose $\eta = 4 \sqrt{1 - \frac{1}{2\xi(s+1)}}$. By using (\ref{eq:part1-done}), (\ref{eq:part2-done}), and (\ref{eq:part3-done}), we achieve the following.
%
\begin{align} \nonumber
\mathbb{E}[\mathscr{T}_{i}(T)] 
&\stackrel{(\ast)}{\leq} 1 + J(\gamma) \Upsilon_{T} 
+  \lceil{T(1-\gamma)\rceil} \left( \left \lceil {\frac{16 \xi (s+1) \log{m_{T}^{\gamma}}}{ (\frac{\Delta_{\min}}{s w_{\max}})^{2} }} \right \rceil \gamma^{-\frac{1}{1-\gamma}} + D \right) \\ \nonumber
&+ \sum_{t \in \Gamma(\gamma)}^{} {\Bigg[} \sum_{\mathbf{M}^{\gamma,D}[v_{1}]=1}^{\left\lceil\frac{1}{1-\gamma}\right\rceil} \dots \sum_{\mathbf{M}^{\gamma,D}[v_{|\mathcal{I}(\mathbf{x}_{t}^{\ast})|}]=1}^{\left\lceil\frac{1}{1-\gamma}\right\rceil}  
\sum_{\mathbf{M}^{\gamma,D}[u_{1}]=\ell}^{\left\lceil\frac{1}{1-\gamma}\right\rceil} \dots \sum_{\mathbf{M}^{\gamma,D}[u_{|\mathcal{I}(\mathbf{x}_{t})|}]=\ell}^{\left\lceil\frac{1}{1-\gamma}\right\rceil} \\ \nonumber
&\hspace{85mm} 2 s^{p} \left\lceil \frac{\log m_{t}^{\gamma}}{\log{(1+\eta)}} \right\rceil^{p} (m_{t}^{\gamma})^{-p} {\Bigg]} \\ \nonumber
&\leq 1 + J(\gamma) \Upsilon_{T} +  \lceil{T(1-\gamma)\rceil}
\left( \left \lceil {\frac{16 \xi s^{2} w_{\max}^{2} (s+1) \log{m_{T}^{\gamma}}}{ \Delta_{\min}^{2} }} \right \rceil \gamma^{-\frac{1}{1-\gamma}} + D \right) \\ 
&\hspace{55mm}+ 2 s^{p}  \left\lceil\frac{1}{1-\gamma}\right\rceil^{2s} \sum_{t \in \Gamma(\gamma)}^{} 
\left\lceil \frac{\log m_{t}^{\gamma}}{\log{(1+\eta)}} \right\rceil^{p} (m_{t}^{\gamma})^{-p},
\end{align}
where $(\ast)$ follows from $\mathbf{M}_{t}^{\gamma,D}[i] \leq m_{t}^{\gamma} \leq \left\lceil\frac{1}{1-\gamma}\right\rceil$, $\forall i \in [N]$, $\forall t \in [T]$. We can control the sum in the last term as follows. By choosing $k = (1-\gamma)^{-1}$, we have
\begin{align} \nonumber
\sum_{t \in \Gamma(\gamma)}^{} 
\left\lceil \frac{\log m_{t}^{\gamma}}{\log{(1+\eta)}} \right\rceil^{p} (m_{t}^{\gamma})^{-p} 
&\leq k + \sum_{t = k}^{T} \left\lceil \frac{\log m_{k}^{\gamma}}{\log{(1+\eta)}} \right\rceil^{p} (m_{k}^{\gamma})^{-p} \\ \nonumber
&\leq k + \left\lceil \frac{\log m_{k}^{\gamma}}{\log{(1+\eta)}} \right\rceil^{p} \frac{T}{(m_{k}^{\gamma})^{p}} \\
&\leq \frac{1}{1-\gamma} + \left\lceil \frac{\log \frac{1}{1-\gamma}}{\log{(1+\eta)}} \right\rceil^{p} \frac{T (1-\gamma)^{p}}{(1-\gamma^{\frac{1}{1-\gamma}})^{p}}.
\end{align}
Hence, the expected regret is upper bounded as
\begin{align} \nonumber
    \mathcal{R}_{T}(\mathcal{X})
    &\leq \Delta_{\max} \sum_{i = 1}^{N} \mathbb{E} [\mathscr{T}_{i}(T)] \\ \nonumber
    &\leq 
    {\Bigg[} 1 + J(\gamma) \Upsilon_{T} +  \lceil{T(1-\gamma)\rceil} 
    \left( \left \lceil {\frac{16 \xi s^{2} w_{\max}^{2} (s+1) \log{m_{T}^{\gamma}}}{ \Delta_{\min}^{2} }} \right \rceil \gamma^{-\frac{1}{1-\gamma}} + D \right) 
    \\ 
    &\hspace{25mm}+ 2 s^{p} \left\lceil\frac{1}{1-\gamma}\right\rceil^{2s} 
    \left( \frac{1}{1-\gamma}
    + \left\lceil \frac{\log \frac{1}{1-\gamma}}{\log{(1+\eta)}} \right\rceil^{p} \frac{T (1-\gamma)^{p}}{(1-\gamma^{\frac{1}{1-\gamma}})^{p}} \right) {\Bigg]} N \Delta_{\max}.
\end{align}
%
\end{proof}

\subsection{Additional Information and Experiments regarding Synthetic Dataset}
\label{app:AddInfo-Synthetic}
%
\subsubsection{Expected Instantaneous Rewards}
\label{App:ExpInstRew}
%
In \textbf{Fig. \ref{fig:MeanRewards-Synthtic}}, we depict the changes in the expected instantaneous reward over time for each base arm in our synthetic dataset. As we see, there are $3$ change points where the expected instantaneous reward of at least one base arm changes abruptly.
\begin{figure}[t!]
    \centering
    \includegraphics[width=0.7\textwidth]{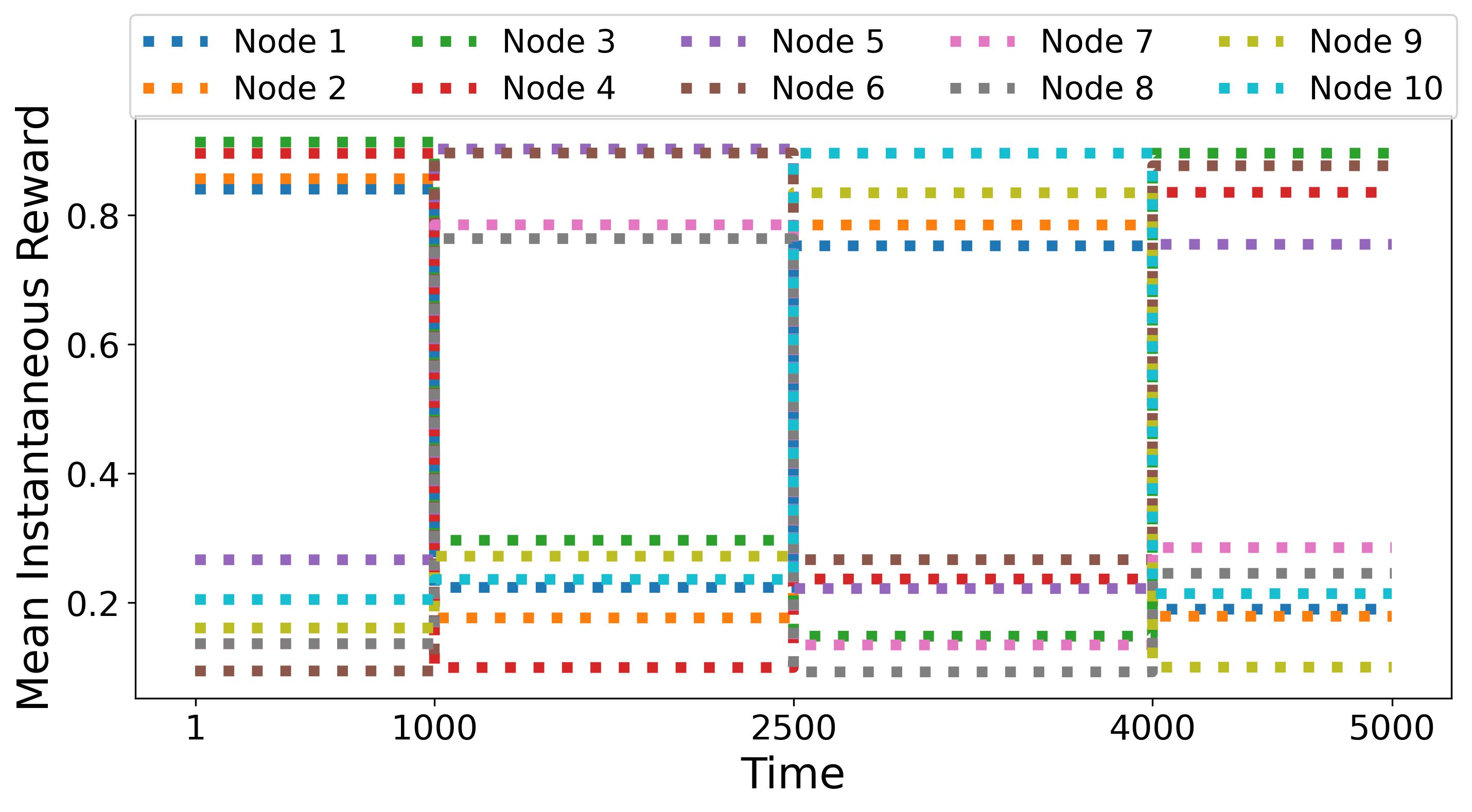}
    \caption{Evolution of the base arms' expected instantaneous reward for the synthetic experiment.}
    \label{fig:MeanRewards-Synthtic}
\end{figure}

\subsubsection{Adaptation to the Changes in the Environment}
\label{App:AdaptationtoChanges}
To further analyze the performance of our algorithm in our experiments using synthetic data, we define the \textit{optimality ratio} for the model during each stationary period. Let $\mathcal{I}(\mathbf{x}) = \left\{ i \in [N] ~|~ \mathbf{x}[i] \neq 0 \right\}$ be the \textit{index set} of a decision vector $\mathbf{x} \in \mathcal{X}$. For the $i$-th stationary period $T_{i} \subseteq [T]$, the optimality ratio of a given policy is calculated as $(\sum_{t \in T_{i}}^{} \sum_{i \in \mathcal{I}(\mathbf{x}_{t})}^{} \mathbbm{1}\{i \in \mathcal{I}(\mathbf{x}_{t}^{\ast})\}) / (\sum_{t \in T_{i}}^{} |\mathcal{I}(\mathbf{x}_{t}^{\ast})|)$. In words, the optimality ratio of a given policy for each stationary period is the ratio of the number of selected base arms by that policy that belong to the optimal super arm in that stationary period over the number of selected base arms by oracle during that stationary period.

\textbf{Fig. \ref{fig:Opt-Ratio}} shows the optimality ratio of the agent over different stationary periods by following the NDC-SEM and SEM-UCB policies. We can observe that our algorithm closely follows the super arm choice pattern of the oracle, which means that it can quickly adapt to changes in the environment. On the other hand, SEM-UCB cannot always adapt to sudden changes in the environment. We particularly consider SEM-UCB in this analysis to show that, although SEM-UCB learns the structural dependencies in the network, it fails in learning the optimal decision vector in the presence of delay and non-stationarity. 
\begin{figure}[t!]
    \centering
     \begin{subfigure}[t]{0.6\textwidth}
         \centering
         \includegraphics[width=1\textwidth]{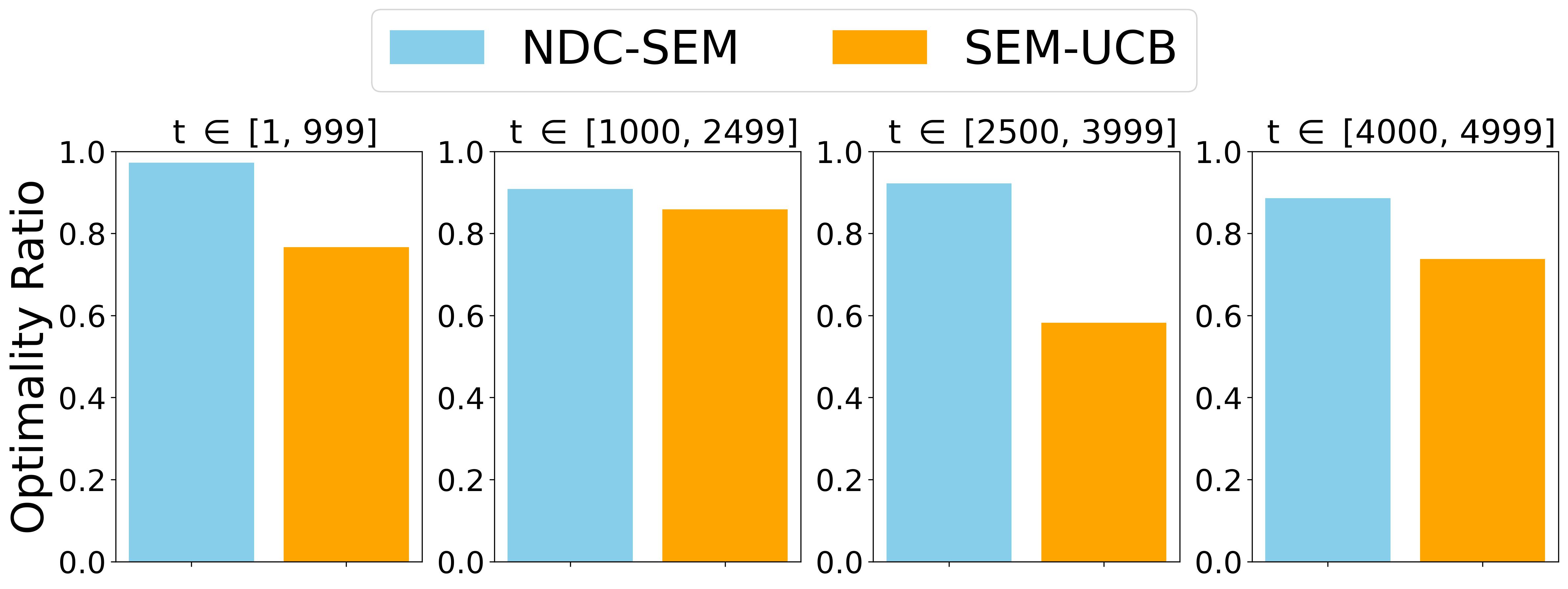}
         \label{subfig:optimality1}
     \end{subfigure}
    \\
     \begin{subfigure}[t]{0.6\textwidth}
         \centering
         \includegraphics[width=1\textwidth]{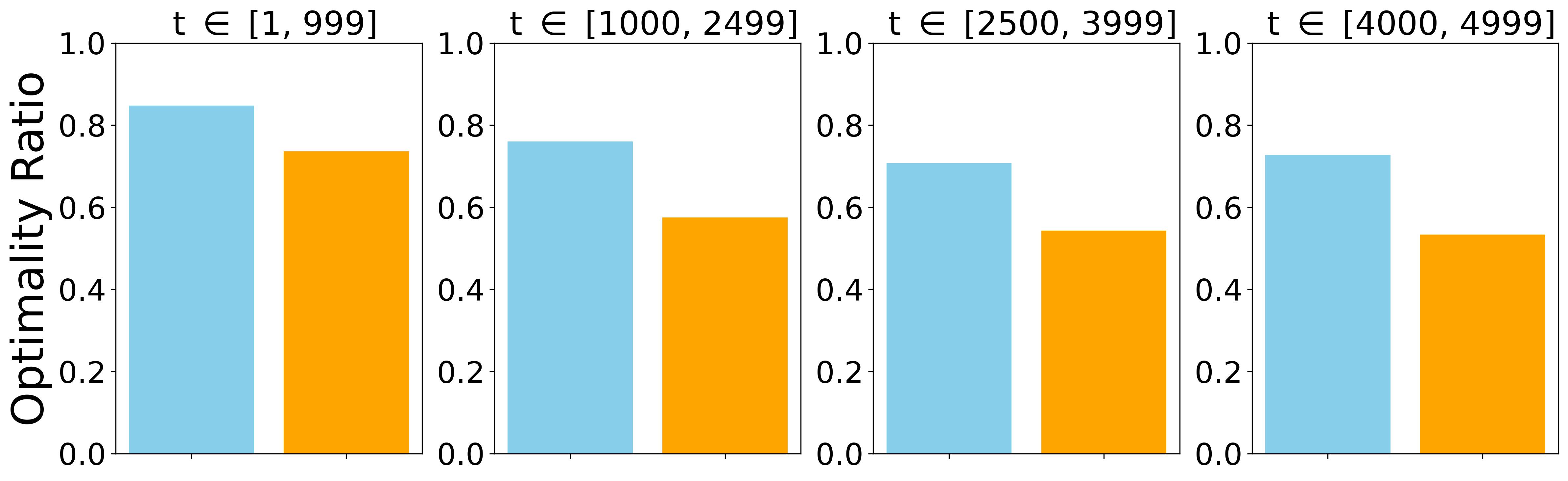}
         \label{subfig:optimality2}
     \end{subfigure}
    \\
     \begin{subfigure}[t]{0.6\textwidth}
         \centering
         \includegraphics[width=1\textwidth]{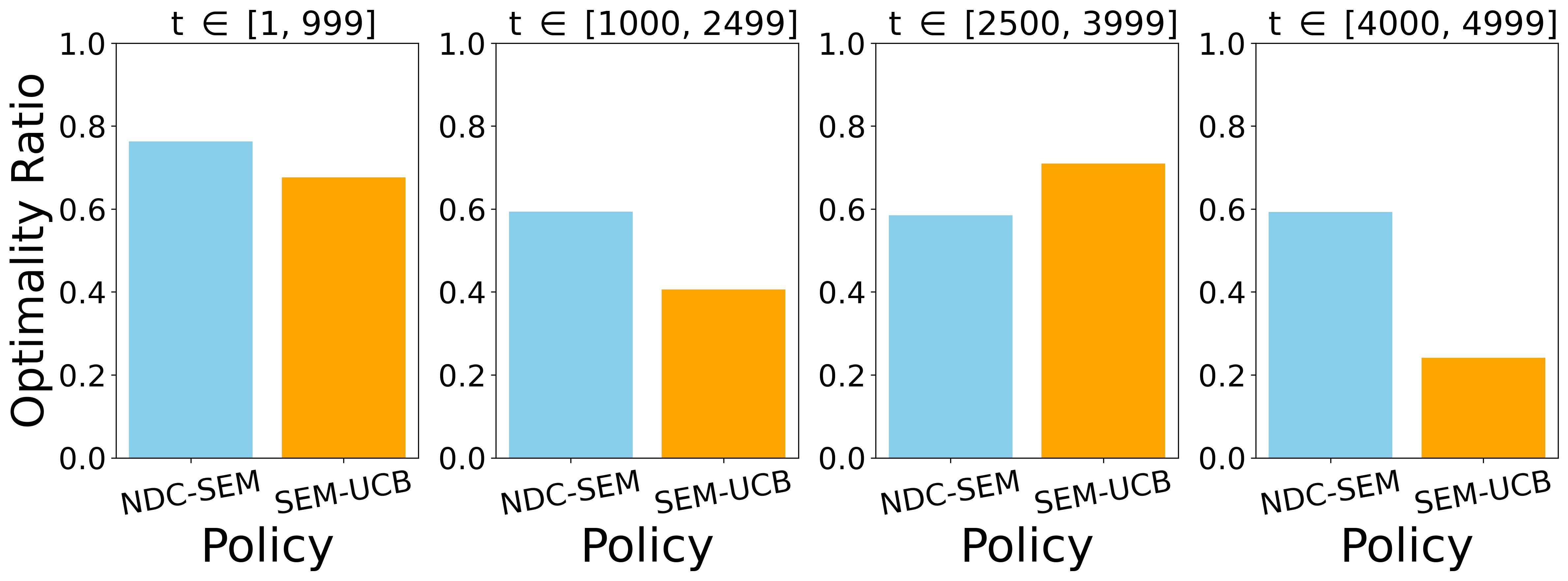}
         \label{subfig:optimality3}
     \end{subfigure}
    \caption{Optimality ratio of NDC-SEM vs. SEM-UCB for delay $D \in \{50, 200, 400\}$ from top to bottom.}
    \label{fig:Opt-Ratio}
\end{figure}

The tuned parameters of the NDC-SEM algorithm in our experiments are listed in \textbf{Table \ref{table:PolicyParams}}.
\renewcommand{\arraystretch}{1.1}
\renewcommand{\tabcolsep}{2mm}
\begin{table}[b!]
\caption{Parameters of the NDC-SEM policy in the experiments.}
\label{table:PolicyParams}
{
\begin{center}
\begin{tabular}{l | l l} 
    \hline
    Experiment & \multicolumn{2}{c}{Parameters} \\
    \hline
    Synthetic (D=50) & $\gamma = 0.985$ & $\xi = 1e-10$ \\
    \hline
    Synthetic (D=200) & $\gamma = 0.985$ & $\xi = 1e-06$ \\
    \hline
    Synthetic (D=400) & $\gamma = 0.985$ & $\xi = 1e-18$\\
    \hline
    Covid-19 & $\gamma = 0.85$ & $\xi = 0.1$ \\
    \hline
    \end{tabular}
    \end{center}
    }     
\end{table}

\subsection{Additional Information and Experiments regarding Covid-19 Dataset
}
\label{sec:AddInfo-covid19}
%
\subsubsection{Abbreviations of the Regions in Italy}
\label{App:NameAbbreviations}
\textbf{Table \ref{tab:AbbList}} lists the original names of the $21$ regions in Italy that we consider in our numerical experiments, together with the corresponding abbreviations.
\renewcommand{\arraystretch}{1.1}
\renewcommand{\tabcolsep}{1.2mm}
\begin{table}[t!]
\centering
\caption{List of regions in Italy and the corresponding abbreviations. \vspace{2mm}}
\label{tab:AbbList}
\scalebox{0.8}{
\begin{tabular}{|l|*{4}{c|}} \toprule
{Abbreviation}
& \makebox[7em]{Region Name}
\\ \midrule
ABR & Abruzzo \\\hline
BAS & Basilicata \\\hline
 CAL & Calabria \\\hline
 CAM & Campania \\\hline
 EMR & Emilia-Romagna \\\hline
 FVG & Friuli Venezia Giulia \\\hline
 LAZ & Lazio \\\hline
 LIG & Liguria \\\hline
 LOM  & Lombardia \\\hline
 MAR  & Marche \\\hline
 MOL & Molise \\\hline
 PAB  & Provincia Autonoma di Bolzano \\\hline
 PAT  & Provincia Autonoma di Trento \\\hline
 PIE  & Piemonte \\\hline
 PUG  & Puglia \\\hline
 SAR  & Sardegna / Sardigna \\\hline
 SIC  & Sicilia \\\hline
 TOS  & Toscana \\\hline
 UMB  & Umbria \\\hline
 VDA  & Valle d'Aosta / Vallée d'Aoste \\\hline
 VEN  & Veneto \\ \bottomrule
\end{tabular}
}
\end{table}

\subsubsection{Overall Daily New Cases of Covid-19 Infection}
%
Italy has been severely affected by the Covid-19 pandemic. In April $2020$, the country showed the highest death toll in Europe.  \textbf{Fig. \ref{fig:OrgCovidData}} depicts the overall daily new cases of $21$ regions in Italy for the considered time interval in our numerical experiments. This figure shows the original daily records before the pre-processing of the dataset in our experiment.
\begin{figure}[ht!]
  \centering
  \includegraphics[width=0.7\textwidth]{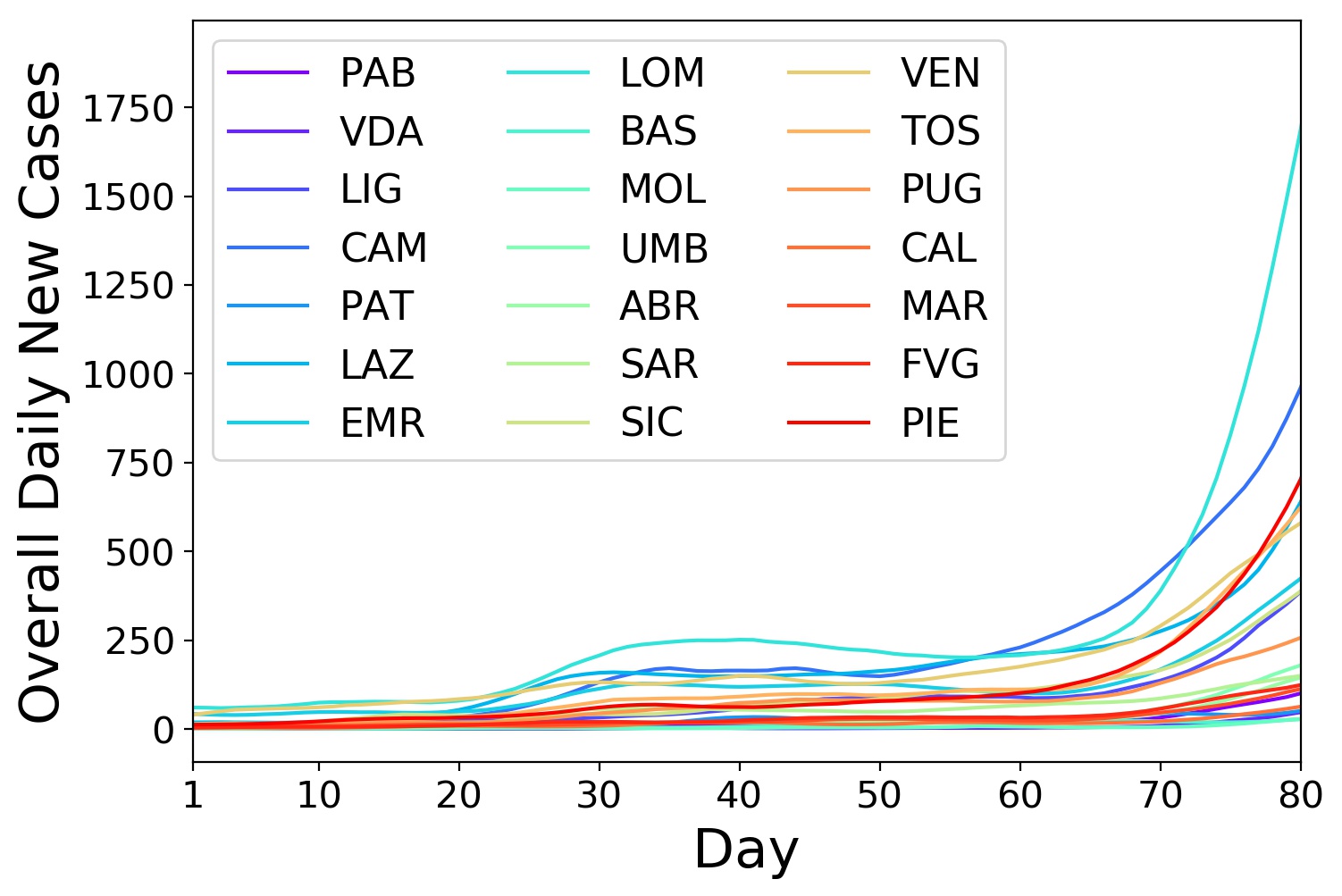}
  \caption{Overall daily new cases of Covid-19 for different regions in Italy during the study period.}
\label{fig:OrgCovidData}
\end{figure}

\subsubsection{Distribution of Region-Specific Daily New Cases}
\label{App:Region-SpecificDist}
The Covid-19 dataset includes only the region's overall daily new cases. Thus, to apply our algorithm, we need to infer the distribution of region-specific daily new cases for each region. To this end, we follow the approach proposed by \cite{Nourani22:LCS} and use the data corresponding to the period from $20$ April to $3$ June, $2020$, to estimate the underlying distributions of the region-specific daily new cases using a kernel density estimation. In particular, from $18$ May to $3$ June, all places for work and leisure activities were opened, and traveling within regions was permitted while traveling between regions was forbidden \cite{bull2021italian}. Consequently, during this period, there are no causal effects amongst the regions' overall daily new cases. In addition, according to google mobility data \cite{nouvellet2021reduction}, from $20$ April to $18$ May, the movement was increasing within the regions while a travel ban between the regions was still imposed. We sample from the aforementioned estimated distributions to create the region-specific daily cases for each region. Then, we apply a $7$-day moving average to the overall and region-specific cases.

\subsubsection{Expected Region-Specific Daily New Cases}
%
In \textbf{Fig. \ref{fig:MeanRewards-Covid}}, we show the trend of the regions' expected instantaneous reward over time in our experiment with the Covid-19 dataset. Note that this figure corresponds to the pre-processed Covid-19 data used in our experiment.
\begin{figure}[b!]
  \centering
  \includegraphics[width=0.65\textwidth]{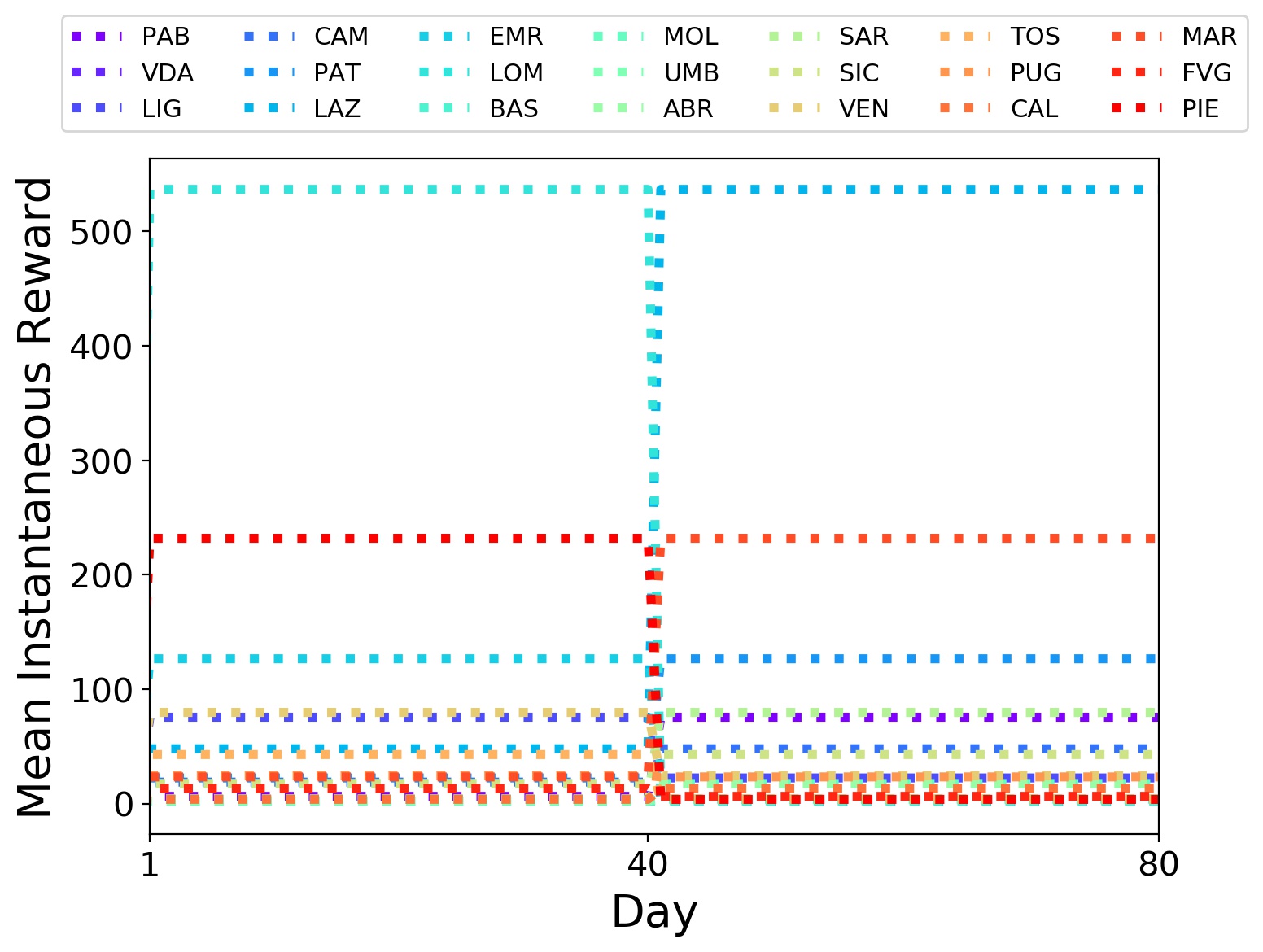}
  \caption{Evolution of the expected region-specific daily new cases for each region over time (corresponding to the pre-processed data).}
\label{fig:MeanRewards-Covid}
\end{figure}

\subsubsection{Hyperparameter Tuning for NDC-SEM Algorithm}
\label{App:HyperparameterTuning}
We simultaneously tuned $\gamma$ and $\xi$ by performing a grid search over the sets $\{0.5, 0.6, 0.65, 0.7, 0.75, 0.8, 0.85, 0.9, 0.95, 0.99\}$ and $\{1e-8, 3e-8, 7e-8, 1e-7, 3e-7, 7e-7, 1e-6, 3e-6, 7e-6, 1e-5, 3e-5, 7e-5, 1e-4, 3e-4, 7e-4, 1e-3, 3e-3, 7e-3, 1e-2, 7e-2, 1e-1, 7e-1\}$, respectively. To that end, we ran the algorithm with each pair of parameters and chose parameters that resulted in the highest estimated cumulative overall reward. To estimate the overall rewards, we used the final estimated adjacency matrix.

\subsubsection{Comparison of Cumulative Overall Reward}
\label{App:comparison}
\textbf{Fig. \ref{fig:EstCollectedOverallRew}} shows the trend of estimated cumulative overall reward for NDC-SEM and SEM-UCB algorithms. The estimated overall rewards are calculated using the final estimated adjacency matrix. As we see, after the change point at day $40$, NDC-SEM performs better than SEM-UCB due to the fact that it considers the effects of non-stationarity and the delay. Consequently, the selected regions by NDC-SEM yield higher accumulated overall rewards than those by SEM-UCB.
\begin{figure}[t!]
    \centering
    \includegraphics[width=0.6\textwidth]{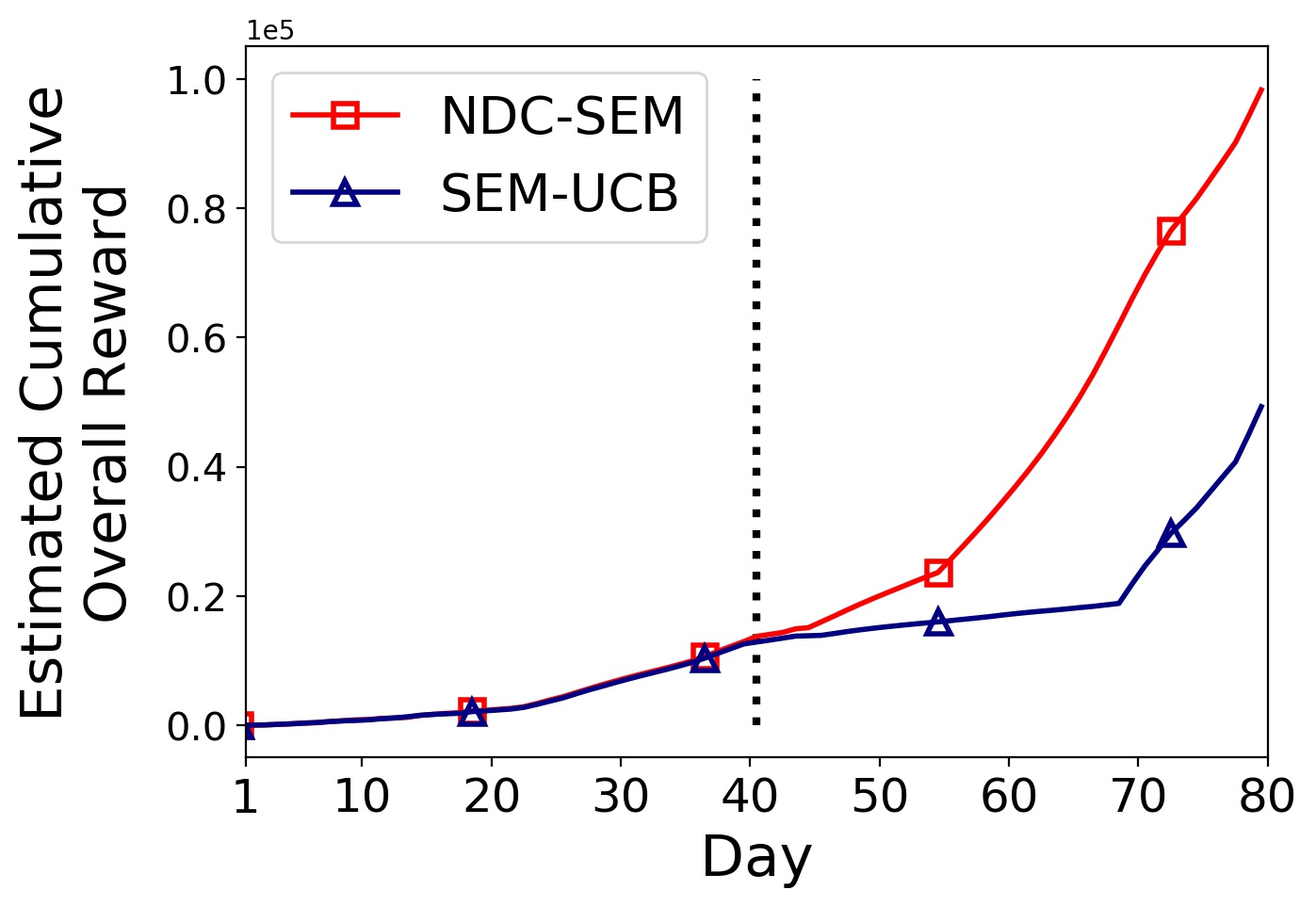}
    \caption{Trend of estimated cumulative overall reward. The vertical dotted line shows the change point at day $40$.} 
    \label{fig:EstCollectedOverallRew}
\end{figure}

%
\bibliographystyle{IEEEbib}
\bibliography{references}
\end{document}